%% file: neurips26_submission.tex
\documentclass{article}

\PassOptionsToPackage{numbers,compress}{natbib}
\usepackage[preprint]{neurips_2026}

\usepackage[utf8]{inputenc}
\usepackage[T1]{fontenc}
\usepackage{hyperref}
\usepackage{url}
\usepackage{booktabs}
\usepackage{amsfonts}
\usepackage{microtype}
\usepackage{xcolor}
\usepackage{amsmath,amssymb}
\usepackage{array}
\usepackage{enumitem}

\usepackage{amsmath,amsfonts,amssymb,amsthm,array}
\usepackage{caption}
\usepackage{graphicx}
\usepackage{booktabs} 
\usepackage{float}
\usepackage{enumitem}
\usepackage{nicefrac}

\usepackage{enumitem}
\usepackage{algorithm}
\usepackage{algorithmic}
\input{math_commands.tex}

\input{macros}

\usepackage{graphicx}

\usepackage{tikz}

\theoremstyle{plain}

\theoremstyle{remark}

\usepackage{mdframed} 
\usepackage{thmtools}

\usepackage[flushleft]{threeparttable} 

\usepackage{caption}
\usepackage{multirow}
\usepackage{colortbl}
\definecolor{bgcolor}{rgb}{0.8,1,1}
\definecolor{bgcolor2}{rgb}{0.8,1,0.8}
\definecolor{niceblue}{rgb}{0.0,0.19,0.56}

\usepackage{hyperref}
\hypersetup{colorlinks,linkcolor={blue},citecolor={blue},urlcolor={blue}}

\usepackage{sidecap}

\definecolor{shadecolor}{gray}{0.9}
\declaretheoremstyle[
headfont=\normalfont\bfseries,
notefont=\mdseries, notebraces={(}{)},
bodyfont=\normalfont,
postheadspace=0.5em,
spaceabove=1pt,
mdframed={
  skipabove=8pt,
  skipbelow=8pt,
  hidealllines=true,
  backgroundcolor={shadecolor},
  innerleftmargin=4pt,
  innerrightmargin=4pt}
]{shaded}

\usepackage{xspace}

\title{\bf Gradient Clipping Beyond Vector Norms: A Spectral Approach for Matrix-Valued Parameters}

\newcommand{\affmark}[1]{\textsuperscript{#1}}

\author{%
Alexander Yukhimchuk\affmark{1}\thanks{Equal contribution} \quad
Mladen Kolar\affmark{1,2} \quad
Martin Takáč\affmark{1} \quad
Sayantan Choudhury\affmark{1}\footnotemark[1]
\\[0.5em]
{\normalfont\small
\affmark{1} MBZUAI \quad
\affmark{2} University of Southern California
}
}

\begin{document}

\maketitle

\begin{abstract}
    Gradient clipping is a standard safeguard for training neural networks under noisy, heavy-tailed stochastic gradients; yet, most clipping rules treat all parameters as vectors and ignore the matrix structure of modern architectures. We show empirically that data outliers often amplify only a small number of leading singular values in layer-wise gradient matrices, while the rest of the spectrum remains largely unchanged. Motivated by this phenomenon, we propose \emph{spectral clipping}, which stabilizes training by clamping singular values that exceed a threshold while preserving the singular directions. This framework generalizes classical gradient norm clipping and can be easily integrated into existing optimizers. We provide a convergence analysis for non-convex optimization with spectrally clipped SGD, yielding the optimal $\gO\left( K^{\nicefrac{2 - 2 \alpha}{3 \alpha - 2}} \right)$ rate for heavy-tailed noise. To minimize hyperparameter tuning, we introduce layer-wise adaptive thresholds based on moving averages or sliding-window quantiles of the top singular values. Finally, we develop efficient implementations that clip only the top $r$ singular values via randomized truncated SVD, avoiding full decompositions for large layers. We demonstrate competitive performance across synthetic heavy-tailed settings and neural network training tasks. 
\end{abstract}

\vspace{-0.3cm}
\section{Introduction}
\vspace{-0.2cm}
    Optimization algorithms lie at the core of deep learning~\citep{lecun2015deep,choudhury2024remove,mishchenko2023prodigy,shi2023ai}. As the cost of training large-scale models~\citep{achiam2023gpt,berahas2016multi} continues to grow, the demand for more effective and efficient optimization methods has become increasingly critical~\citep{guo2025deepseek,nguyen2017sarah,schaipp2023momo,takavc2013mini,li2022sp2}.

    The parameters of modern deep learning models are predominantly matrix-valued~\citep{vaswani2017attention, lecun2015deep, he2016deep}. Despite this, most existing optimization algorithms operate by implicitly vectorizing all parameters, thereby discarding intrinsic algebraic and geometric structures. Consequently, widely used algorithms such as Stochastic Gradient Descent (\algname{SGD})~\cite{robbins1951stochastic}, \algname{RMSProp}~\citep{tieleman2012lecture}, \algname{AdaDelta}~\cite{zeiler2012adadelta},
    \algname{Adam}~\citep{kingma2014adam}, and \algname{AdamW}~\citep{loshchilov2017decoupled, bjorck2021understanding} are fundamentally designed for vector-valued parameters. Recently, \citet{jordan2024muon} advocated the use of different optimizers for matrix-valued parameters, enabling the exploitation of structural properties that are lost under vectorization. This line of work has revealed a rich and largely unexplored design space for optimization algorithms in deep learning~\citep{shen2025convergence, chang2025convergence, an2025asgo, pethick2025training, shah2025practical,he2017distributed, choudhury2026muonnesterovmomentumheavytailed}.

    Despite this progress, the use of a matrix structure has largely been limited to determining the direction of the optimization step~\citep{jordan2024muon, an2025asgo, gupta2018shampoo, pethick2025training, vyas2024soap}. Other critical components for stabilizing training, such as gradient clipping~\cite{mikolov2012statistical, pascanu2013difficulty}, continue to treat parameters as vectors, ignoring their inherent matrix structure. In contrast, this work proposes a spectral approach to gradient clipping that explicitly leverages the structure of weight parameters.
    
    Specifically, we provide empirical evidence showing that the leading singular values of gradient matrices tend to increase in the presence of data noise. Motivated by this observation, we propose clipping the top singular values of matrix-valued gradients when they exceed a predefined threshold, thereby improving training stability.
    
    \vspace{-0.2cm}
    \subsection{Background and Related Works}
    \vspace{-.1cm}
    
        Most deep learning tasks reduce to solving the minimization problem
        $\min_{\vx \in \R^d} f(\vx)$
        and \algname{SGD}~\citep{robbins1951stochastic} is the most common approach for solving it. The update rule of \algname{SGD} is given by
        $\vx_{k+1} = \vx_k - \eta_k \vg_k$,
        where $\eta_k > 0$ is the step size and $\vg_k$ is an unbiased estimator of the gradient of $f$ at $\vx_k$, i.e., $\Exp{\vg_k \mid \vx_k} = \nabla f(\vx_k)$. There are many studies showing the convergence of \algname{SGD} for solving convex and non-convex problems~\citep{gorbunov2020unified, ghadimi2013stochastic}. In practice, \algname{SGD} works well when the distribution of gradient noise is light-tailed. To address rare but potentially very large stochastic gradients, one often replaces $\vg_k$ by
        $\text{clip}(\vg_k,\tau) \eqdef \min\left\{1,\frac{\tau}{\|\vg_k\|}\right\}\vg_k$~\citep{koloskova2023revisiting}$,$
        which yields the update
        \begin{eqnarray}\label{eq:sgd_clip}
            \vx_{k+1} & = & \vx_k - \eta_k \text{clip}(\vg_k, \tau).
        \end{eqnarray}
        Here $\tau > 0$ is the clipping threshold, which ensures that $\|\text{clip}(\vg_k,\tau)\| \leq \tau$ for all $k$.
        Gradient clipping was established as a practical remedy for exploding gradients in recurrent neural networks~\citep{mikolov2012statistical, pascanu2013difficulty} and remains a standard stabilization tool in modern training pipelines. Beyond this classical motivation, a recent theoretical literature has clarified when clipping helps and what its limitations are. \citet{koloskova2023revisiting} showed that clipped \algname{SGD} introduces a threshold-dependent stochastic bias and provided tight convergence guarantees, while \citet{zhang2020improved, qian2021understanding} analyzed clipping-based methods under relaxed smoothness. Very recent work also extends norm clipping beyond the Euclidean setting: \citet{pethick2025generalized} studies generalized gradient norm clipping under non-Euclidean $(L_0,L_1)$-smoothness and instantiates this viewpoint in \algname{Clipped Scion}. Additional related works about heavy tailed gradient noise and spectral computation could be found in Appendix~\ref{missing_theory}.

    \vspace{-0.2cm}
    \subsection{Main Contributions}
    \vspace{-.1cm}
    The main contributions of this work are summarized below.

    $\bullet$ \textbf{Novel Gradient Clipping Technique.} We propose a new method for gradient clipping, which we call spectral clipping (Section \ref{sec:spectral_clip}), that utilizes the matrix structure of the parameters' gradients and clamps the larger singular values beyond a certain threshold. This scheme captures the state-of-the-art gradient norm clipping method as a special case (Section \ref{sec:gradientnorm_clip}).

    $\bullet$ \textbf{Convergence Analysis.} We analyze the proposed spectral clipping scheme for non-convex optimization problems of the form \eqref{eq:matrix_min}. Under the bounded $\alpha$-th moment assumption (Assumption \ref{assume:BV}), we present our core theoretical contributions in Theorems \ref{theorem:SGD_clipping} and \ref{theorem:main_theorem}. The analysis in Theorem \ref{theorem:SGD_clipping} is tight, recovering state-of-the-art convergence rates for standard \algname{SGD} in the absence of clipping. Furthermore, we explicitly characterize the bias term $\frac{1}{K} \sum_{k=0}^{K-1} \Delta_k$ incurred by the clipping operation. Finally, in Theorem \ref{theorem:main_theorem}, we establish an optimal convergence rate of $\mathcal{O} \left(K^{\nicefrac{2 - 2\alpha}{3\alpha-2}}\right)$ under the heavy-tailed noise framework of \citet{zhang2020adaptivemethodsgoodattention}.

    $\bullet$ \textbf{Adaptive Clipping Strategies.} In Section \ref{sec:adaptive_clip}, we provide adaptive strategies for selecting clipping thresholds $\tau_k$~\eqref{eq:SGD_sclipped} based on exponential moving averages~\eqref{eq:ema_clip} or quantiles~\eqref{eq:quantile_clip1} of the top singular values of gradient matrices. This allows us to tackle noise and adapt to gradient scales during training. We show the effectiveness of these strategies in experiments.

    $\bullet$ \textbf{Efficient Implementation of Spectral Clipping.} For the efficient implementation of our proposed method to train large neural networks, we discuss three important aspects:
    \begin{enumerate}
        \item \textbf{Faster Implementation.} In Section \ref{sec:spectral_clipping_power_iteration}, we describe how truncated SVD can be used to implement spectral clipping of gradient matrices. We introduce Algorithm \ref{alg:spectral_clipping_rsvd}, which is more efficient compared to implementing spectral clipping with full SVD and does not compromise accuracy (see Figure \ref{fig:SVDvstruncatedSVD_accuracy}). 
        \item \textbf{Layerwise Adaptive Clipping Scheme.} In Section \ref{sec:layerwise_ema}, we propose the layerwise adaptive clipping scheme, where we maintain different clipping thresholds for different layers of the neural network.
    \end{enumerate}
    $\bullet$ \textbf{Numerical Experiments.} 
    In Section~\ref{sec:numerical_exp}, we present several numerical experiments that demonstrate the effectiveness of the proposed spectral clipping technique. In particular, we show that (i) \algname{SGDM} with spectral clipping outperforms \algname{SGDM} with norm clipping on the CIFAR-10 task in Section~\ref{sec:resnet}; (ii) on the language modeling task in Section~\ref{sec:gpt2_fineweb}, spectral clipping outperforms norm clipping for \algname{SGDM}, \algname{Adam} and \algname{Muon}; and (iii) \algname{SGDM} with spectral clipping outperforms \algname{Adam} in Section~\ref{sec:mlp}.

    \begin{figure*}[t!]
        \centering
        \begin{subfigure}[t]{0.56\linewidth}
            \centering
            \includegraphics[width=\linewidth]{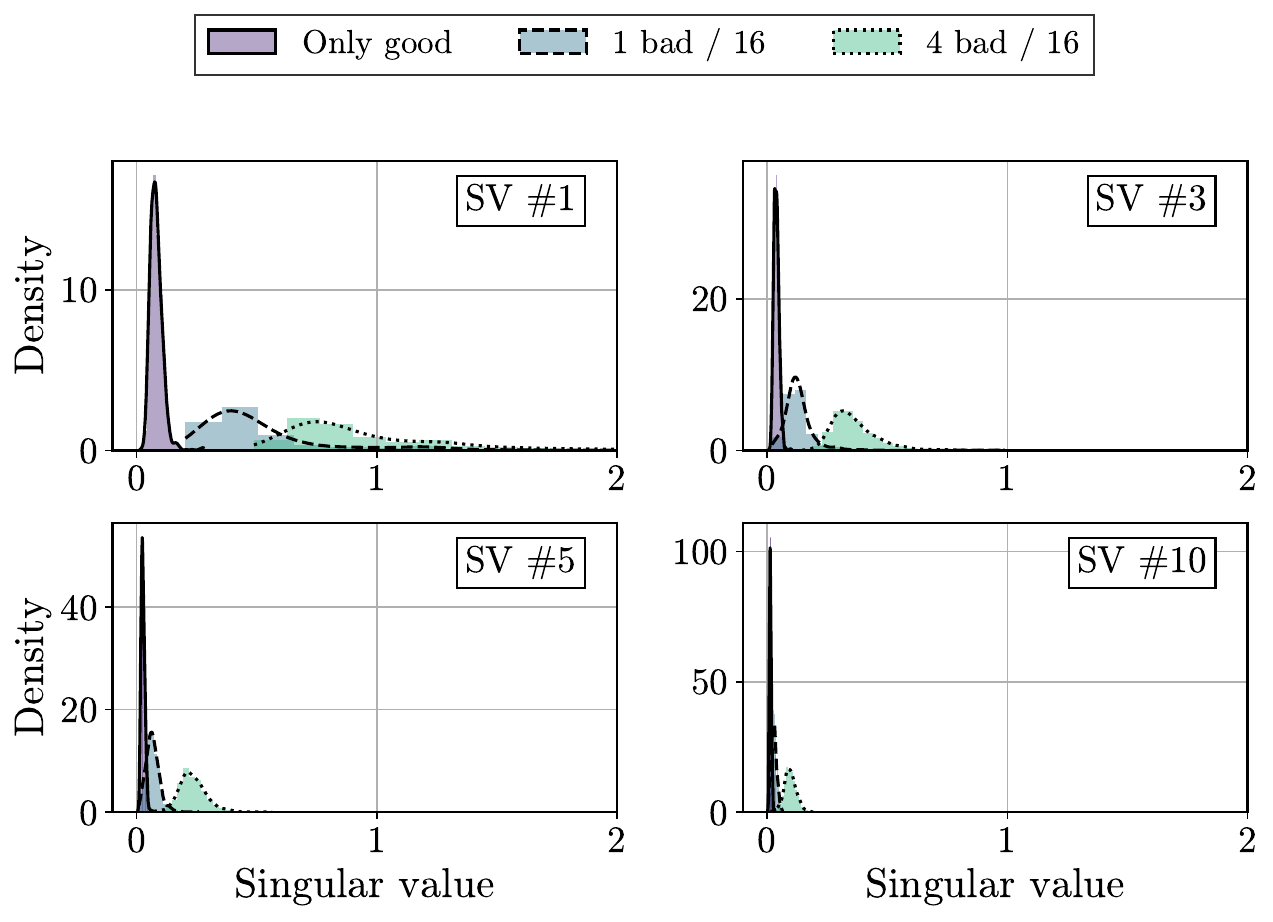}
            \caption{Singular value distributions for selected SVs}
            \label{fig:sv_distribution_kde_left}
        \end{subfigure}\hfill
        \begin{subfigure}[t]{0.38\linewidth}
            \centering
            \includegraphics[width=\linewidth]{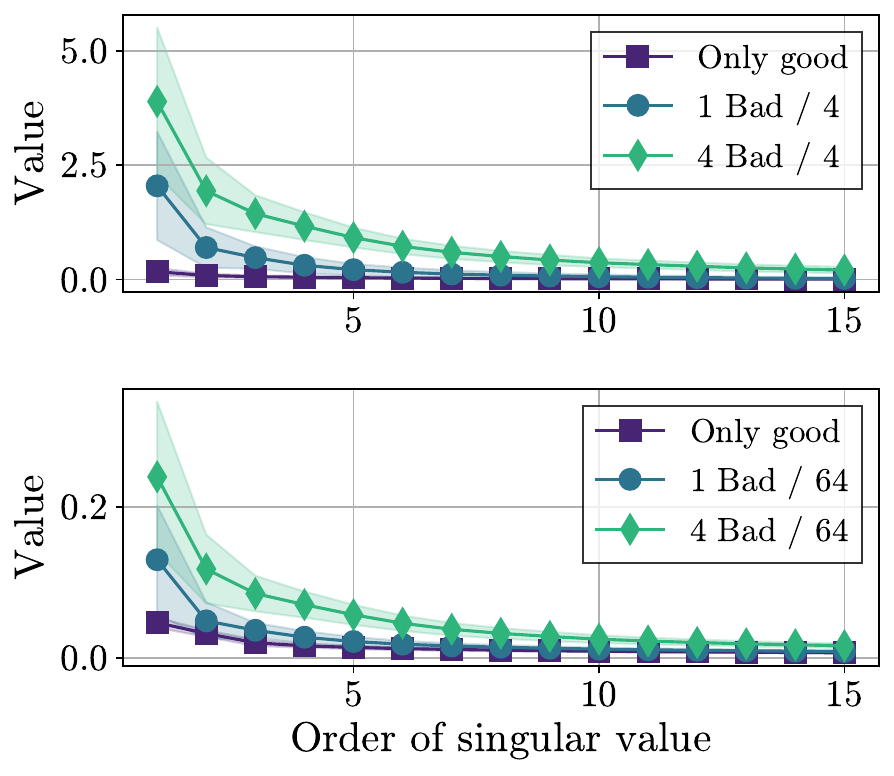}
            \caption{Top-15 singular values}
            \label{fig:sv_distribution_lines_right}
        \end{subfigure}
        \caption{\small Spectral profiles under controlled token replacement of the first feed-forward matrix \(W_{\text{MLP, 1}}\) in layer 5 of GPT‑2 Small. (a) Singular value distributions for selected SVs at batch size 16. (b) Top‑15 singular values for batch sizes 4 and 64. Initially, the batch contains only good samples with \(\tau_{\text{good}} = 0.95\); we replace one and four tokens with bad samples with \(\tau_{\text{bad}} = 10^{-7}\). Adding bad samples shifts the distribution and makes it broader.}
        \vspace{-.5cm}
        \label{fig:sv_distribution_with_kde_layer_5_grid}
    \end{figure*}

\vspace{-.3cm}
\section{Gradient Clipping beyond Vector Norms}
\vspace{-0.2cm}
    In this section, we first conduct an experiment in which we artificially add noise and check the influence of this noise on the spectrum of gradient matrices. Based on our observations in this experiment, we suggest a novel clipping technique in the second part. Finally, we show that our method is more general than the norm clipping method discussed in \eqref{eq:sgd_clip}.

    \vspace{-.2cm}
    \subsection{Singular Values of Gradient Matrices}
    \vspace{-.1cm}
    Again, all the works mentioned above treat the gradients as vectors, while each layer of a neural network has a matrix structure, and their corresponding gradients are also matrices. To gain a deeper understanding, we study the gradients of these layer-wise gradients. Here, we consider the pretrained GPT-2 model $p_{\theta}(\cdot \mid \cdot)$ ~\citep{radford2019language}. This model is already trained, and given a sequence of tokens $s_{[1:t-1]}$, it assigns a high probability to a \textit{good} token; i.e., $p_{\theta}(s_t \mid s_{[1: t-1]})$ is high if $s_t$ is good, and a lower probability to a \textit{bad} token; i.e., $p_{\theta}(s_t \mid s_{[1: t-1]})$ is low if $s_t$ is a bad token~\cite{sutskever2014sequence}. Now we implement backpropagation on $p_{\theta}$ to compute the gradients of the parameters for both the good and bad sets of tokens.

    In Figure~\ref{fig:sv_distribution_with_kde_layer_5_grid}, we plot the distribution of the largest singular values of gradient matrices computed from mini-batches containing different numbers of bad tokens. When all tokens in the mini-batch are good, the singular values of the gradient matrix remain small. As the number of bad tokens in the mini-batch increases, the top singular values increase significantly.
    
    We perform this analysis for each layer of GPT-2 and observe similar trends across all layers (see Appendix~\ref{app:singular_values_overview} for additional details). These observations lead us to consider a clipping technique that truncates the singular values of gradient matrices in order to address noise in modern deep learning.
    
\vspace{-.2cm}
\subsection{Spectral Clipping of Gradient Matrices}\label{sec:spectral_clip}
\vspace{-.1cm}
    Motivated by the empirical behavior observed in the previous section, we treat model parameters as matrices and clip the singular values of the corresponding stochastic gradient matrices. To elaborate on our technique, we consider the matrix-parameter optimization problem
    \begin{equation}\label{eq:matrix_min}
    \textstyle
        \min_{\mX \in \R^{m \times n}} f(\mX),
    \end{equation}
    where $f: \R^{m \times n} \to \R$. To solve this problem, we propose the update rule 
    \begin{equation}\label{eq:SGD_sclipped}
    \textstyle
        \mX_{k+1} = \mX_k - \eta_k \gC_{\tau_k}(\mG_k).
    \end{equation}
    where $\gC_{\tau_k}$ is an operator that clamps the singular values of $\mG_k$ to a threshold $\tau_k > 0$. For example, if the SVD of the gradient matrix is 
    $\mG_k = \mU_k \text{diag}(\bm{\sigma}_k)\mV_k^\top$
    (where $d \coloneqq \min\{m,n\}, \mU_k\in\R^{m\times d}, \mV_k\in\R^{n\times d}$ and $\bm{\sigma}_k\in\R_+^{d}$ contain the singular values) then $\gC_{\tau_k}(\mG_k) \eqdef \mU_k \text{diag}(\clamp_{\tau}(\bm{\sigma}_k)) \mV_k^\top,
    $ where the coordinates of the vector $\clamp_{\tau}(\bm{\sigma}_k)$ are given by
    \begin{equation}\label{eq:def_clamp}
    \textstyle
        \clamp_{\tau_k} (\bm{\sigma}_k)[i] = \begin{cases}
            \bm{\sigma}_k[i] & \text{if } \bm{\sigma}_k[i] \leq \tau_k \\
            \tau_k & \text{otherwise}.
        \end{cases} 
    \end{equation}
    We call this technique spectral clipping. The operator $\gC_{\tau_k}$ leaves the singular directions $\mU_k, \mV_k$ unchanged and ensures $\sigma_{\max}(\gC_{\tau_k}(\mG_k)) \leq \tau_k$ and $\| \gC_{\tau_k}(\mG_k) \|_F \leq \sqrt{d} \tau_k$.
    As we observe, the top singular values of the gradient matrices spike in the presence of noise; this clipping technique can help in practice. 
    
    \vspace{-.2cm}
    \subsection{Gradient Norm Clipping as a Special 
    Case}\label{sec:gradientnorm_clip}
    \vspace{-.1cm}
    We next show that standard gradient norm clipping is a special case of spectral clipping. Norm clipping treats parameters as vectors and updates using \eqref{eq:sgd_clip}. Observe that any
    vector $\vg_k \in \R^d$ can be viewed as a matrix in $\R^{d\times 1}$, and its SVD is $\vg_k = \mU_k \text{diag}(\bm{\sigma}_k)\mV_k^\top$, with $\mU_k = \frac{\vg_k}{\| \vg_k\|} \in \R^{d}$, $\bm{\sigma}_k = \| \vg_k\| \in \R$, and $\mV_k = 1 \in \R$. Then from \eqref{eq:def_clamp}, we have $\text{clamp}_{\tau_k} \left(\| \vg_k \| \right) = \min \{ \| \vg_k\|, \tau_k \}$ and thus $\gC_{\tau_k}(\vg_k) = \frac{\vg_k}{\| \vg_k\|} \min \left\{\| \vg_k\|, \tau_k \right\}$. Plugging this into~\eqref{eq:SGD_sclipped} gives
    \begin{equation*}
    \textstyle
        \vx_{k+1} = \vx_k - \eta_k \gC_{\tau_k}(\vg_k) = \vx_k - \eta_k \frac{\vg_k}{\| \vg_k\|} \min \left\{\| \vg_k\|, \tau_k \right\} = \vx_k - \eta_k \min \left\{ 1, \frac{\tau_k}{\| \vg_k\|} \right\} \vg_k,
    \end{equation*}
    which coincides with the norm-clipping update~\eqref{eq:sgd_clip} (with $\tau_k = \tau$). Therefore, spectral clipping generalizes gradient norm clipping, and \emph{any analysis of our method captures that of gradient norm clipping as a special case}. 
\vspace{-0.3cm}
\section{Convergence Guarantees}
\vspace{-.2cm}
    In this section, we establish convergence guarantees for \algname{SGD} with spectral clipping \eqref{eq:SGD_sclipped}. For this section, we consider only the constant spectral clipping threshold, i.e., $\tau_k = \tau$ in \eqref{eq:SGD_sclipped}. We consider the matrix optimization problem \eqref{eq:matrix_min}, and we make some standard assumptions on the function $f$ and the stochastic gradient matrix $\mG_k$. 
    \begin{assumption}\label{assume:smooth}
    We assume $f$ is $L$-smooth i.e. for all $\mX, \mY \in \R^{m \times n}$ we have
        $$f(\mX) \leq f(\mY) + \langle \nabla f(\mY), \mX - \mY \rangle + \frac{L}{2} \| \mX - \mY \|_F^2.$$
    \end{assumption}
    This is a standard assumption in the optimization literature for the analysis of optimization algorithms~\citep{nesterov2004introductory, gorbunov2020stochastic, koloskova2023revisiting}. 
    Moreover, we assume that the stochastic gradient estimates satisfy the bounded $\alpha$-moment assumption~\citep{zhang2020adaptivemethodsgoodattention, hubler2024gradient, chezhegov2024clipping, choudhury2026muonnesterovmomentumheavytailed}, also called $p$-BCM.
    \begin{assumption}\label{assume:BV}
        We assume the gradient $\mG_k$ computed at each iteration $k$ is unbiased i.e. $\Exp{\mG_k \mid \mX_k} = \nabla f(\mX_k)$ and has bounded $\alpha$-moment i.e.
        $\Exp{\| \mG_k - \nabla f(\mX_k) \|_F^\alpha \mid \mX_k } \leq \sigma^\alpha$
        for some $\alpha \in (1, 2]$ and $\sigma > 0$.
    \end{assumption}
    For $\alpha = 2$, Assumption~\ref{assume:BV} reduces to the classical bounded-variance assumption~\citep{lan2012optimal, dekel2012optimal}, and it is extensively used in optimization to study different stochastic algorithms. 
    However, the bounded variance assumption can be restrictive in modern applications. Many recent empirical studies suggest that the stochastic gradient noise in image classification~\citep{simsekli2019tail, battash2024revisiting}, large language models~\citep{zhang2020adaptivemethodsgoodattention, ahn2023linear}, and reinforcement ~\citep{garg2021proximal} tasks follows a heavy-tailed distribution. Assumption \ref{assume:BV} accommodates such heavy-tailed behavior with $\alpha < 2$~\citep{zhang2020adaptivemethodsgoodattention, mohammadi2015estimating}. Under this assumption, we can show that $\Exp{\| \gC_{\tau}(\mG_k)\|_F^2 \mid \mX_k } \leq 2 \|\nabla f(\mX_k) \|_F^2 + 2^{4 - \alpha} d^{1 - \nicefrac{\alpha}{2}} \tau^{2-\alpha} \sigma^\alpha$ (see Lemma \ref{lemma:bound_2_moment}, Appendix \ref{sec:convergence_analysis}) for $d = \min \{ m, n\}$ and use this to prove the following bound on the iterates of \algname{SGD} with spectral clipping~\eqref{eq:SGD_sclipped}.
    \begin{theorem}\label{theorem:SGD_clipping}
        Suppose Assumptions \ref{assume:smooth} and \ref{assume:BV} hold, and let $f_* \coloneqq \inf_{\mX} f(\mX) > -\infty$. Then for any $K \geq 1$, \algname{SGD} with spectral clipping~\eqref{eq:SGD_sclipped} and a constant step size $\eta_k = \eta \eqdef \frac{1}{4 L \sqrt{K}}$ satisfies 
        \begin{equation}\label{eq:SGD_clipping}
            \textstyle
            \min_{0 \leq k \leq K-1} \Exp{\| \nabla f (\mX_k) \|_F^2}
            \leq \frac{16 L (f(\mX_0) - f_*)}{\sqrt{K}} + \frac{C_{\alpha, d} \tau^{2 - \alpha} \sigma^\alpha}{8 \sqrt{K}} + \frac{2}{K} \sum_{k = 0}^{K-1} \Delta_k
        \end{equation}
        where $\Delta_k \eqdef \Exp{\left\| \nabla f(\mX_k) - \Exp {\mathcal{C}_{\tau_k}(\mG_k) \mid \mX_k } \right\|_F^2}$ and $C_{\alpha, d} \eqdef 2^{4 - \alpha} d^{1 - \nicefrac{\alpha}{2}}$ with $d = \min \{ m, n\}$.
    \end{theorem}
    The first term in~\eqref{eq:SGD_clipping} is the standard optimization error for $L$-smooth objectives. The second term depends on $\tau^{2 - \alpha} \sigma^\alpha$ and captures the effect of stochastic noise after clipping.  For $\alpha<2$, this term reflects how clipping controls the contribution of heavy-tailed gradient noise to the convergence rate. The last term is dependent on the clipping threshold $\tau$ and measures the bias introduced by the nonlinear operator $\mathcal{C}_{\tau_k}$. Using conditional unbiasedness, we can rewrite the bias as $        \nabla f(\mX_k) - \E[\mathcal{C}_{\tau_k}(\mG_k)\mid \mX_k] = \E[\mG_k - \mathcal{C}_{\tau_k}(\mG_k)\mid \mX_k]$,  
    so $\Delta_k$ is small whenever clipping is rarely active (i.e., $\mathcal{C}_{\tau_k}(\mG_k)\approx \mG_k$ with high probability). A similar phenomenon of bias is observed in \algname{SGD} with gradient norm clipping~\cite{koloskova2023revisiting}. When clipping is disabled by taking $\tau_k=\infty$, we have $\mathcal{C}_{\tau_k}(\mG_k)=\mG_k$ and hence $\Delta_k = \Exp{\left\| \nabla f(\mX_k) - \Exp {\mG_k \mid \mX_k } \right\|_F^2} = 0$.

    For $\alpha = 2$, note that the second term on the right hand side of \eqref{eq:SGD_clipping} becomes independent of the clipping threshold $\tau$. Since clipping can introduce bias through $\Delta_k$ but does not improve the variance dependent second term of \eqref{eq:SGD_clipping} when $\alpha=2$, the best bound is obtained by disabling clipping ($\tau_k=\infty$), thereby recovering the optimal $\mathcal{O}(\nicefrac{1}{\sqrt{K}})$ convergence rate for non-convex minimization with \algname{SGD}~\citep{robbins1951stochastic,ghadimi2013stochastic}. However, when gradients have heavy-tailed noise, i.e., $\alpha < 2$, we need clipping to control both the second and third terms of \eqref{eq:SGD_clipping} and achieve a better rate of convergence. We provide a detailed proof of Theorem \ref{theorem:SGD_clipping} in Appendix \ref{sec:proof_theorem:SGD_clipping}.

    The bias due to a finite clipping threshold in \eqref{eq:SGD_clipping}, i.e., $\frac{2}{K} \sum_{k = 0}^{K-1} \Delta_k$, is unavoidable. However, we can choose the threshold $\tau$ depending on the noise level $\sigma$ to prove convergence. 

    \begin{theorem}\label{theorem:main_theorem}
        Suppose Assumptions \ref{assume:smooth} and \ref{assume:BV} hold, and let $f_* \coloneqq \inf_{\mX} f(\mX) > -\infty$. Then, for any $K \geq 1$, \algname{SGD} with spectral clipping \eqref{eq:SGD_sclipped}, a constant step size $\eta_k = \eta \leq \nicefrac{1}{4dL}$ and clipping threshold $\tau \geq \max \{ 2, 8 \sigma \}$ satisfy the following
        \begin{equation}\label{eq:main_theorem_eq1}
        \textstyle
            \min_{0 \leq k \leq K-1} \Exp{ \phi \left( \| \nabla f(\mX )\|_F \right)}  
            \leq \frac{4 (f(\mX_0) - f_*)}{\eta K} + 32 \tau^{-2(\alpha - 1)} \sigma^{2 \alpha} + 8 \eta d L \tau^{2 - \alpha} \sigma^{\alpha}
        \end{equation}
        where $\phi(t) \eqdef \min\{ t, t^2\}$ and $d = \min \{m, n\}$. Moreover, when the step size $\eta \leq \min \left\{\nicefrac{1}{4dL}, \nicefrac{1}{\tau^\alpha L} \right\}$ and the threshold $\tau \geq \max \left\{ 2, 8 \sigma, \sigma K^{\nicefrac{1}{3 \alpha - 2}} \right\}$, we have 
        \begin{equation}\label{eq:main_theorem_eq2}
        \textstyle
            \hspace{-2mm} \min_{0 \leq k \leq K-1} \Exp{ \phi \left( \| \nabla f(\mX )\|_F \right)} = \mathcal{O} \left( K^{\frac{-2 \alpha + 2}{3 \alpha - 2}} \right).
        \end{equation}
    \end{theorem}
    Theorem \ref{theorem:main_theorem} establishes the convergence of \algname{SGD} with spectral clipping to a stationary point for a non-convex problem under Assumption \ref{assume:BV}. Here we use $\min_{0 \leq k \leq K-1} \Exp{\min \{ \| \nabla f(\mX_k)\|_F, \| \nabla f(\mX_k)\|_F^2 \}}$ as the potential function to measure convergence. \citet{qian2021understanding, zhang2020adaptivemethodsgoodattention, zhang2020improved} also uses the same potential function for the analysis of stochastic clipping algorithms. When the gradients are small, i.e., $\| \nabla f(\mX_k)\|_F \leq \varepsilon \ll 1$, then $\| \nabla f(\mX_k)\|_F^2 \ll \| \nabla f(\mX_k)\|_F $ and $\| \nabla f(\mX_k)\|_F^2$ are the dominant terms. Therefore, in this particularly relevant regime, \eqref{eq:main_theorem_eq2} yields a convergence rate of $\Exp{\| \nabla f(\mX_k)\|_F^2} = \gO\left( K^{\nicefrac{-2 \alpha + 2}{3 \alpha - 2}} \right)$. In particular, for $\alpha = 2$, the right hand side of \eqref{eq:main_theorem_eq2} becomes $\gO \left( K^{\nicefrac{-1}{2}}\right)$, matching the optimal convergence rate for \algname{SGD}~\cite{lan2012optimal, ghadimi2013stochastic}.

    Finally, \citet{zhang2020adaptivemethodsgoodattention} (Theorem~6) proves that the rate
    $\mathcal{O} \left(K^{\nicefrac{-2\alpha+2}{3\alpha-2}}\right)$ is optimal for gradient norm clipping
    under their setting. Since norm clipping is a special case of our spectral clipping strategy,
    this provides evidence that the rate in~\eqref{eq:main_theorem_eq2} is tight for our clipping scheme as well.
    
\vspace{-.3cm}
\section{Adaptive Clipping Strategies}\label{sec:adaptive_clip}
\vspace{-.2cm}
    As discussed in Theorem \ref{theorem:SGD_clipping}, clipping introduces bias; therefore, we need to choose the threshold $\tau$ to handle the noise $\sigma$ and the bias $\frac{1}{K}\sum_{k = 0}^{K-1} \Delta_k$. In Theorem \ref{theorem:main_theorem}, we show that one can choose $\tau$ depending on the noise level $\sigma$ to ensure convergence. However, in practice, we do not have any knowledge of $\sigma$. Moreover, \citet{zhang2020adaptivemethodsgoodattention} shows that the gradient noise distribution changes during the course of training. Therefore, we propose adaptive strategies for choosing the clipping thresholds. We describe our strategies below.
    \vspace{-.1cm}
    \paragraph{EMA of Singular Values.} The clipping threshold $\tau_k$ can be chosen using an exponential moving average of the top singular values of the gradient matrix $\mG_k$, i.e.,
    \begin{eqnarray}
        \hat{\tau}_k & = &  \theta \hat{\tau}_{k-1} + (1 - \theta) \sigma_{\max}(\mG_k), \notag \\
        \tau_k & = & \nicefrac{\hat{\tau}_k}{(1 - \theta^{k+1})} \label{eq:ema_clip}
    \end{eqnarray}
    for some $\theta \in (0, 1)$. In all our experiments, we use $\theta = 0.9$, which places more weight on previous estimates and less on the current gradient. As discussed in the previous section, the top singular values of gradient matrices increase in the presence of noise. Therefore, when there is sudden noise in the gradients, it will not significantly distort our clipping threshold. Similar adaptive ideas were proposed by \citep{zhang2020adaptivemethodsgoodattention} for a norm based clipping method.

    \vspace{-.1cm}
    \paragraph{Quantile of Singular Values.} We also propose a quantile-based adaptive threshold computed over a sliding window of $w$ top singular values, i.e.,
    \begin{eqnarray}
        \tau_k & = & \gQ_q(\gS_{k-1}), \label{eq:quantile_clip1} \\
        \gS_k & = & \left( \gS_{k-1} \setminus \sigma_{\max} \left( \mG_{k - w - 1} \right) \right) \cup \sigma_{\max} \left( \mG_k \right) \label{eq:quantile_clip2}
    \end{eqnarray}
    for some $q \in (0, 1]$. Here $\gS_k$ is the set of the last $w$ top singular values, and $\gQ_q$ denotes the quantile operator at level $q \in [0,1]$. Thus, \eqref{eq:quantile_clip1} sets the clipping threshold based on the $q$-th quantile of the sliding window set $\gS_{k-1}$, and \eqref{eq:quantile_clip2} updates the window by discarding the outdated singular value $\sigma_{\max} \left( \mG_{k - w - 1} \right)$ and adding the new one $\sigma_{\max} \left( \mG_k \right)$. We then use this threshold $\tau_k$ at the $k$-th iteration to clip the raw gradient matrix $\mG_k$ with operator $\gC_{\tau_k}$. Note that our EMA based clipping strategy can still be sensitive to occasional extremely large singular values. However, this quantile-based rule can mitigate this issue by selecting the threshold based on the empirical distribution of the top singular values. In experiments, we use $q = 0.85$ and $w = 100$.
    
    In Figure \ref{fig:trace_reg} (Appendix~\ref{sec:missing_figures}), we compare spectral clipping with different choices of threshold $\tau_k$ for the trace regression problem~\citep{kadri2020partial}
    \begin{eqnarray}\label{eq:trace_reg}
    \textstyle
        \min_{\mX \in \R^{n \times n}} f(\mX) \eqdef \frac{1}{N} \sum_{i = 1}^N \left( \left\langle \mA_i, \mX \right\rangle - \vb_i \right)^2
    \end{eqnarray}
    and plot the average gradient bias $\frac{1}{K} \sum_{k = 0}^{K-1} \Delta_k$ on the $y$-axis. In this example,  we observe that $\frac{1}{K} \sum_{k = 0}^{K-1} \Delta_k$ approaches $0$ when the threshold $\tau_k$ is a sufficiently large constant, but is non-zero when $\tau_k$ is a small constant (i.e. $\leq 0.1$ ). This example shows that we need to tune the clipping threshold $\tau_k$ for better accuracy of the spectral clipping algorithm. Moreover, we also implement our EMA based \eqref{eq:ema_clip} and quantile-based \eqref{eq:quantile_clip2} clipping strategy. We find that these two adaptive clipping strategies are independent of tuning $\tau_k$ and work well in Figure \ref{fig:trace_reg}.

    

\vspace{-0.3cm}
\section{Implementation of Spectral Clipping}
\vspace{-0.2cm}
    In this section, we discuss three aspects to make spectral clipping efficient for training deep neural networks. First, we discuss how spectral clipping can be combined with other optimization algorithms. Then, we propose an efficient clipping procedure using a truncated SVD approximation. Finally, in the third subsection, we suggest using an adaptive threshold for each layer of the deep neural network.

    \vspace{-.2cm}
    \subsection{Spectral Clipping beyond \algname{SGD}}
    \vspace{-.1cm}
    In algorithm~\eqref{eq:SGD_sclipped}, we use spectral clipping for the simple \algname{SGD} update. However, this idea of spectral clipping can be extended to other algorithms like \algname{SGD} with momentum (\algname{SGDM})~\citep{polyak1964some, liu2020improved} or \algname{Adam}~\citep{kingma2014adam}. For training modern deep neural networks, we often use gradient norm clipping along with \algname{SGDM} or \algname{Adam}. Therefore, we can replace gradient norm clipping with spectral clipping for these algorithms. We provide the pseudocode for \algname{SGDM} with EMA based and quantile based spectral gradient clipping in Algorithm \ref{alg:SGDM_ema} and Algorithm \ref{alg:SGDM_quantile} (Appendix \ref{sec:missing_pseudocode}). 
    
    \vspace{-.2cm}
    \subsection{Spectral Gradient Clipping with Truncated SVD}\label{sec:spectral_clipping_power_iteration}
    \vspace{-.1cm}
    Here, we discuss a practical implementation of the spectral clipping method. Recall that the update rule in \eqref{eq:SGD_sclipped} requires applying spectral clipping to the gradient matrix $\mG_k$ at each iteration, which involves knowledge of singular value decomposition (SVD)~\citep{stewart1993early}. For the very large matrices that arise in modern language models, computing a complete SVD at every step is prohibitively expensive~\citep{dongarra2018singular} and therefore impractical.   

    \begin{wrapfigure}{l}{0.52\textwidth}
    \begin{minipage}{0.5\textwidth}
    \begin{algorithm}[H]
    \caption{EMA Spectral Clipping}
    \label{alg:SGDM_ema}
    \begin{algorithmic}[1]
        \REQUIRE $\mX_0$, $\mB_{-1}=0$, $\beta$, $\theta$, $\tau_{-1}, \hat{\tau}_{-1}=0$
        \FOR{$k = 0,\dots,K-1$}
            \STATE Compute $\mG_k$
            \STATE $\mB_k = \beta \mB_{k-1} + (1-\beta)\mathcal{C}_{\tau_{k-1}}(\mG_k)$
            \STATE $\mX_{k+1} = \mX_k - \eta \mB_k$
            \STATE $\hat{\tau}_k = \theta \hat{\tau}_{k-1} + (1-\theta)\sigma_{\max}(\mG_k)$
            \STATE $\tau_k = \nicefrac{\hat{\tau}_k}{1 - \theta^{k+1}}$
        \ENDFOR
    \end{algorithmic}
    \end{algorithm}
    \end{minipage}
    \end{wrapfigure}
    Recently, \citet{zhao2024galore} developed a memory-efficient optimization algorithm called \algname{GaLore} (Gradient Low-Rank Projection) for large model training. Specifically, \algname{GaLore} projects the gradients into a low-rank subspace and performs updates only in that space. To obtain this low-rank projection, \citet{su2025galore} computes a rank-$r$ truncated SVD using a fast randomized SVD algorithm~\citep{halko2011finding}. This motivates an analogous strategy for spectral clipping: instead of computing the entire SVD of the gradient matrix $\mG$, we only compute a rank-$r$ truncated SVD using a randomized method from \citet{halko2011finding} and then clamp the top $r$ singular values of $\mG$.

    \vspace{-.3cm}
    \paragraph{Rank-$r$ Truncated SVD.} For a matrix $\mG \in \R^{m \times n}$ with a full SVD given by $\mG = \sum_{i=1}^{\min\{m,n\}} \sigma_i \vu_i \vv_i^\top = \mU \mathrm{diag}(\bm{\sigma}) \mV^\top,$
    the rank-$r$ truncated SVD retains only the leading $r$ singular values and singular vectors:
    $\mG_r \eqdef \sum_{i=1}^{r} \sigma_i \vu_i \vv_i^\top = \mU[:r] \mathrm{diag}(\bm{\sigma}[:r]) \mV[:r]^\top,$
    where $\mU[:r] \in \R^{m \times r}$ and $\mV[:r] \in \R^{n \times r}$ contain the first $r$ columns of $\mU$ and $\mV$, and $\bm{\sigma}[:r] \in \R^r$ stacks the top $r$ singular values. The matrix $\mG_r$ is the best rank-$r$ approximation to $\mG$ ~\citep{eckart1936approximation}. Randomized SVD~\citep{halko2011finding} efficiently approximates this truncated SVD by constructing a low-dimensional subspace that captures the dominant singular directions and then performing an SVD in that subspace. PyTorch implements this algorithm of \citet{halko2011finding} via the \texttt{torch.svd\_lowrank} function and computes the truncated SVD approximation. Its cost scales as $\mathcal{O}(mn r + (m + n) r^2)$, which is substantially cheaper than a full SVD when $r \ll \min\{m,n\}$. We utilize \texttt{torch.svd\_lowrank} in Algorithm \ref{alg:spectral_clipping_rsvd} for spectral clipping of matrices.

    \begin{wrapfigure}{l}{0.52\textwidth}
    \vspace{-.8cm} 
    \begin{minipage}{0.5\textwidth}
    \begin{algorithm}[H]
    \caption{Clipping via Truncated SVD}
    \label{alg:spectral_clipping_rsvd}
    \begin{algorithmic}[1]
        \REQUIRE Gradient matrix: $\mG \in \R^{m \times n}$, threshold $\tau > 0$, rank $r \in \mathbb{N}$
        \STATE $\overline{\mU}, \overline{\boldsymbol{\sigma}}, \overline{\mV} = \texttt{torch.svd\_lowrank}(\mG, r)$
        \STATE $\overline{\boldsymbol{\sigma}}_{\textsc{extra}, \tau} = \overline{\boldsymbol{\sigma}} - \mathrm{clamp}_{\tau}(\overline{\boldsymbol{\sigma}})$
        \STATE \textbf{return} $\mG - \overline{\mU}\,\mathrm{diag}(\overline{\boldsymbol{\sigma}}_{\textsc{extra}, \tau})\,\overline{\mV}^\top$
    \end{algorithmic}
    \end{algorithm}
    \end{minipage}
    \end{wrapfigure}
    

    In Algorithm~\ref{alg:spectral_clipping_rsvd}, \texttt{torch.svd\_lowrank} first approximates the top $r$ singular values and associated singular vectors of $\mG$ using a randomized SVD routine. For the randomized SVD routine, we use only \texttt{niter = 1} for efficiency. Next, we compute the vector $\overline{\bm{\sigma}}_{\textsc{extra},\tau}$ that captures how much each of the top $r$ singular values exceeds the clipping threshold $\tau$. 
    Finally, we subtract the corresponding rank-$r$ matrix 
    $\overline{\mU}\,\mathrm{diag}(\overline{\bm{\sigma}}_{\textsc{extra},\tau})\,\overline{\mV}^\top$ from $\mG$, which effectively clamps the top $r$ singular values at $\tau$, while leaving smaller singular values and all directions outside the rank-$r$ subspace unchanged.
    
    In practice, we choose $r$ to be small (for our experiments $r = 10$), so that Algorithm~\ref{alg:spectral_clipping_rsvd} computes and clips only the few most dominant singular directions of the gradient matrix, providing substantial computational savings compared to a full SVD while retaining the main benefits of spectral clipping. In Figure \ref{fig:SVDvstruncatedSVD_accuracy} (Appendix~\ref{sec:missing_figures}), we compare the performance of SVD and truncated SVD with $r = 10$ for training a neural network with spectral clipping, and we observe comparable performance on the CIFAR-10 dataset~\citep{krizhevsky2009learning}. \emph{This highlights the efficiency of Algorithm \ref{alg:spectral_clipping_rsvd} without compromising accuracy.}

    \begin{figure*}[!t]
    \centering
    \begin{subfigure}[t]{0.31\textwidth}
        \centering
        \includegraphics[width=\textwidth]{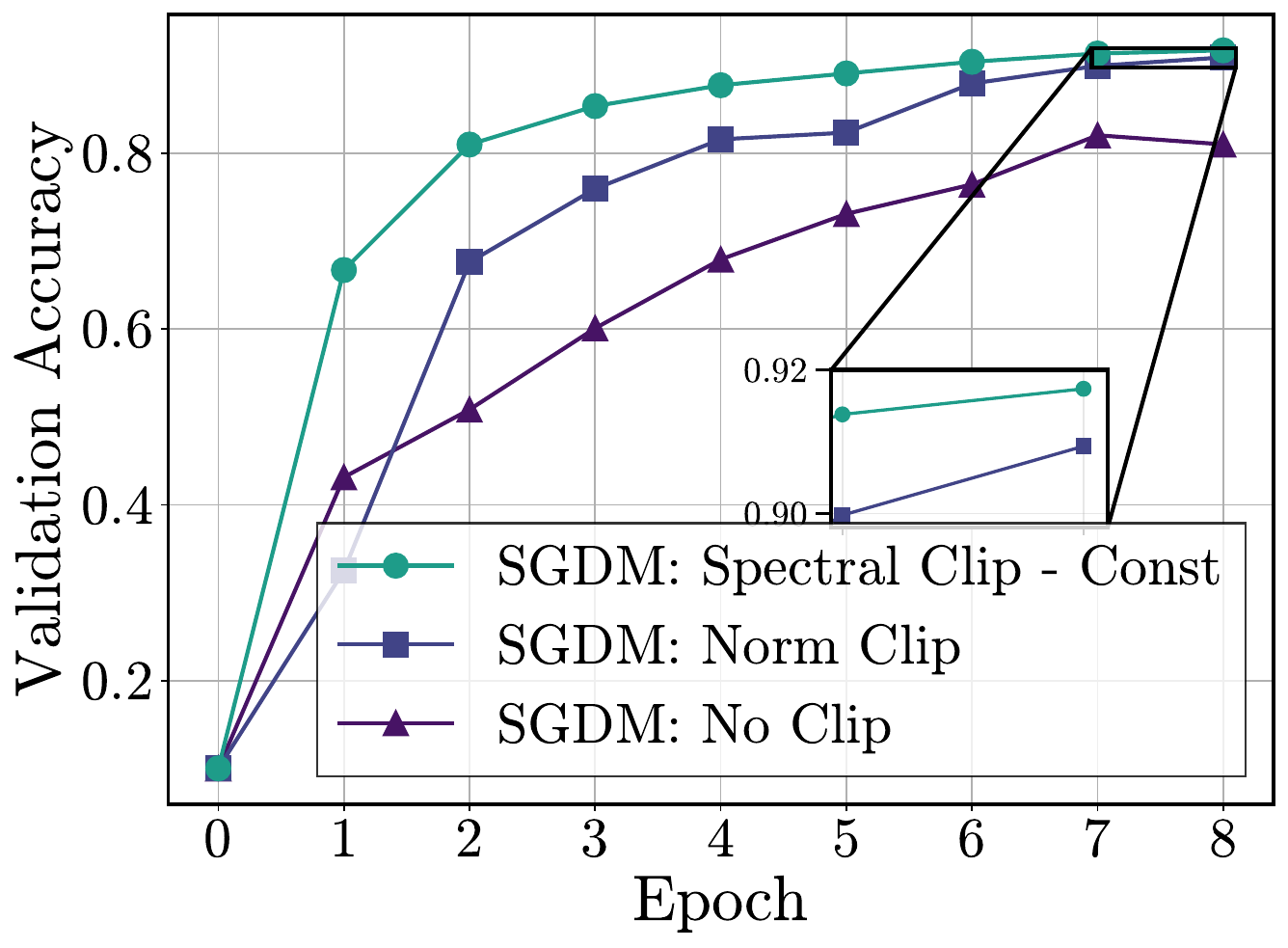}
        \caption{ResNet18 on CIFAR-10}
        \label{fig:cv_sgd_curves}
    \end{subfigure}
    \hfill
    \begin{subfigure}[t]{0.32\textwidth}
        \centering
        \includegraphics[width=\textwidth]{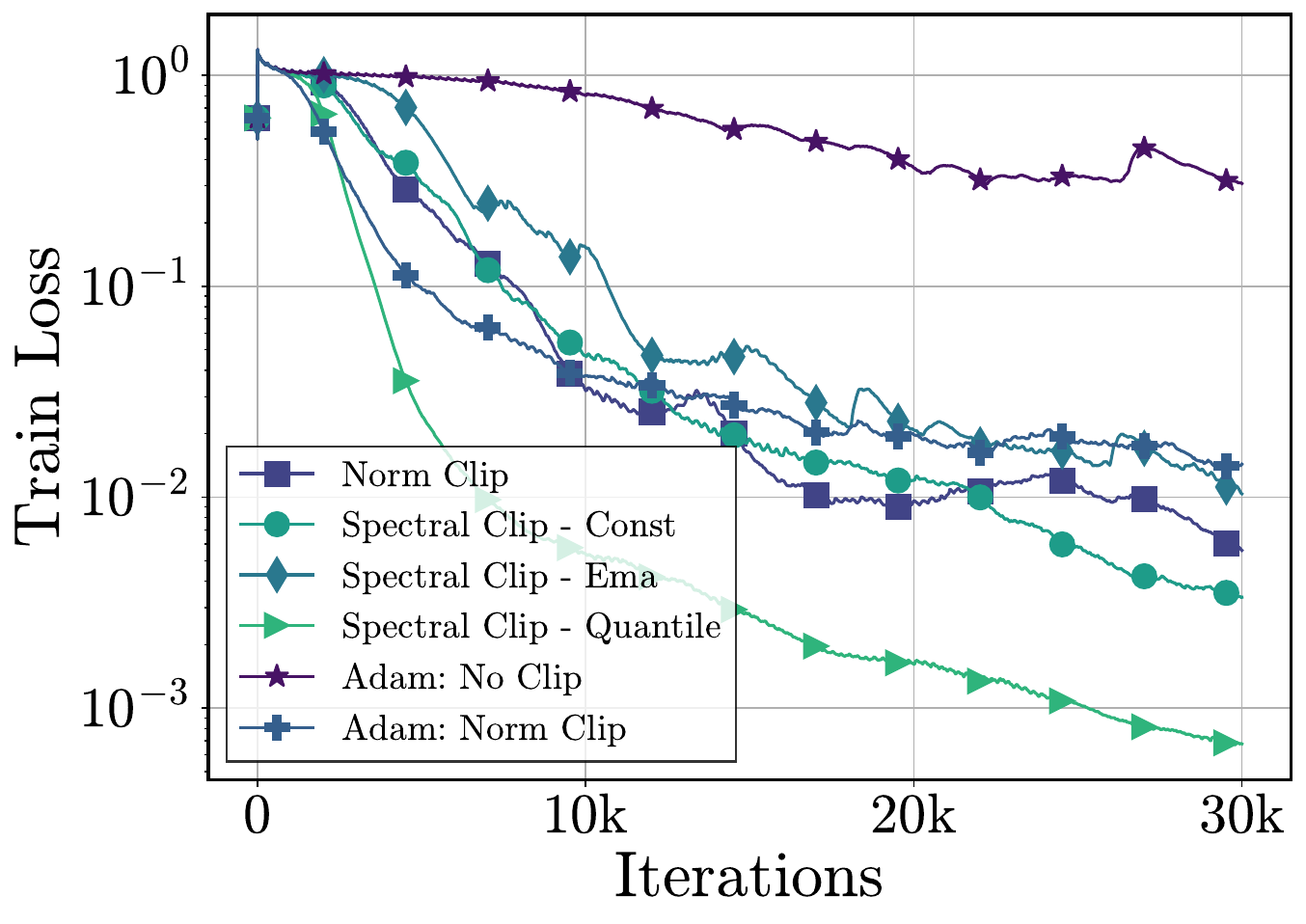}
        \caption{MLP with heavy-tailed noise}
        \label{fig:mlp_sgdm_30k}
    \end{subfigure}
    \hfill
    \begin{subfigure}[t]{0.31\linewidth}
    \centering
    \includegraphics[width=\linewidth]{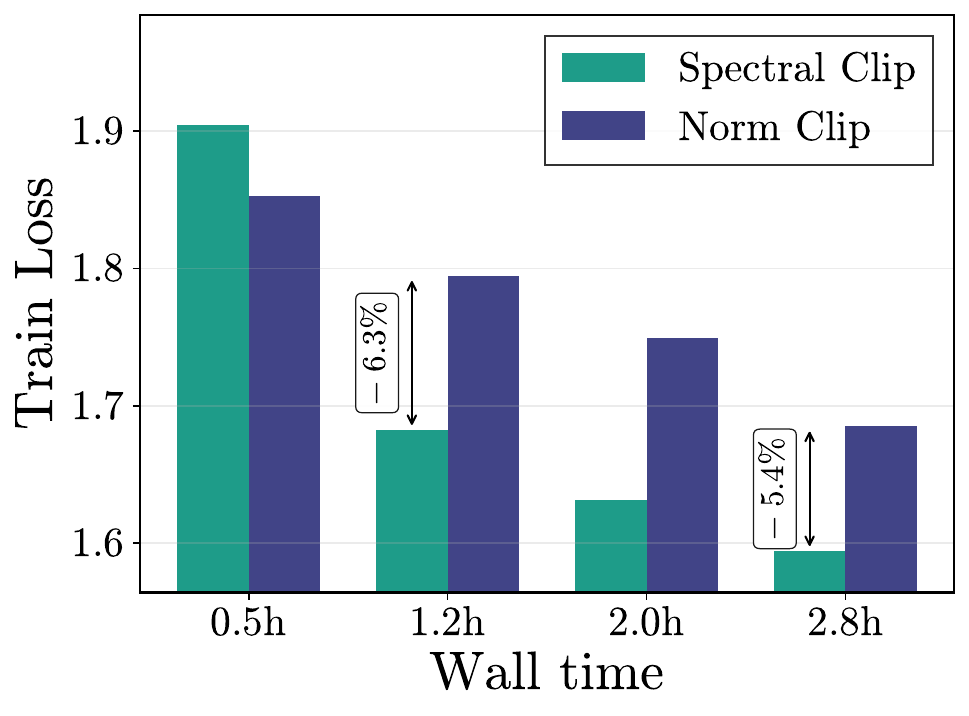}
    \caption{\small Time comparison}
    \label{fig:nanogpt_sgdm_wall_time}
    \end{subfigure}
    \caption{\small In Figure \ref{fig:cv_sgd_curves}, we plot the trajectories of \algname{SGDM} with no clipping, norm clipping and spectral clipping for CV problem. In Figure \ref{fig:mlp_sgdm_30k}, we compare the performance of \algname{SGDM} and \algname{Adam} for different clipping strategies for MLP problem. In Figure~\ref{fig:nanogpt_sgdm_wall_time}, we compare wall-clock time in NanoGPT experiment.
    }
    \vspace{-.5cm}
    \end{figure*}

    \vspace{-.2cm}
    \subsection{Layer-wise Adaptive Clipping}\label{sec:layerwise_ema}
    \vspace{-.1cm}
    In deep neural networks, different layers operate at different scales. Thus, using a single clipping threshold $\tau_k$ at iteration $k$ for the entire network can be suboptimal~\citep{brock2021high, zhang2020adaptivemethodsgoodattention}. This can result in one of two scenarios: (1) over-clipping layers with small singular values, which results in slow optimization, or (2) under-clipping layers with large singular values, thus harming the stability of the algorithm. 
    
    To address this, we propose a layer-wise adaptive spectral clipping scheme in which each layer maintains its own adaptive clipping threshold. Let $\mG_k^{(\ell)}$ denote the gradient of layer $\ell \in \{ 1, \cdots, L \}$ at iteration $k$. Then, for each layer $\ell$, we have an EMA-based clipping threshold $\tau_k^{(\ell)}$ given by $\hat{\tau}_k^{(\ell)} =\theta \hat{\tau}_{k-1}^{(\ell)} + (1 - \theta) \sigma_{\max} \left( \mG_k^{(\ell)} \right), \tau_k^{(\ell)}  = \nicefrac{\hat{\tau}^{(\ell)}_k}{(1 - \theta^{k+1})}$
    or a quantile-based clipping threshold given by $\tau_k^{(\ell)} = \gQ_q \left(\gS^{(\ell)}_{k-1} \right)$ and $\gS_k^{(\ell)} = \left( \gS^{(\ell)}_{k-1} \setminus \sigma_{\max} \left( \mG^{(\ell)}_{k - w - 1} \right) \right) \cup \sigma_{\max} \left( \mG^{(\ell)}_k \right)$. These per-layer thresholds $\{\tau_k^{(\ell)}\}_{\ell=1}^L$ are then used for spectral clipping of layer-wise gradient matrices $\{\mG_k^{(\ell)}\}_{\ell=1}^L$ using Algorithm \ref{alg:spectral_clipping_rsvd}.

    This layer-wise clipping scheme can have several practical advantages. First, it adapts the clipping threshold to each layer of the neural network, allowing aggressive clipping in layers with large singular values while avoiding unnecessary clipping in layers with small singular values. Secondly, the EMA-based and quantile-based updates of the clipping threshold can adapt during training (e.g., phase transitions), thus automatically increasing or decreasing $\tau_k^{(\ell)}$ as the spectrum of $\mG_k^{(\ell)}$ evolves.

\vspace{-.3cm}
\section{Numerical Experiments}\label{sec:numerical_exp}
\vspace{-0.2cm}
    In this section, we compare the proposed spectral clipping method with existing gradient clipping strategies. First, we evaluate the effectiveness of spectral clipping on real-world benchmarks by training ResNet-18~\citep{he2015deepresiduallearningimage} on the CIFAR-10 image classification task. Second, we show that the effectiveness of spectral clipping extends to LLM training tasks. Finally, we conduct controlled experiments on a synthetic regression task using a four-layer multilayer perceptron (MLP), where we explicitly inject low-rank, heavy-tailed noise into the stochastic gradient matrices and demonstrate the superior performance of spectral clipping. Hyperparameter tuning details for all experiments are provided in Appendix~\ref{sec:missing_experiments}. Our codes are available at \url{https://github.com/katcinskiy/spectral-clipping}.

\vspace{-.2cm}
\subsection{Training ResNet18 on CIFAR-10}\label{sec:resnet}
\vspace{-.1cm}

    
    We evaluate spectral clipping on a standard computer vision benchmark by training ResNet-18~\citep{he2015deepresiduallearningimage} on the CIFAR-10 dataset~\citep{krizhevsky2009learning}. Here we consider \algname{SGDM} optimizer with three gradient handling strategies: no clipping, standard gradient norm clipping, and spectral clipping with a constant threshold.
    We train each model for 8 epochs using the AirBench setup~\citep{jordan2024airbench}.

    For convolutional layers, we apply spectral clipping after reshaping each gradient tensor into a matrix.
    
    As shown in Figure~\ref{fig:cv_sgd_curves}, spectral clipping with a constant threshold achieves the best validation accuracy for \algname{SGDM} on CIFAR-10 using 8 epochs of training while also exhibiting a faster initial convergence rate.


    \vspace{-.3cm}

    \begin{figure}[t]
        \centering
        \begin{subfigure}[t]{0.32\linewidth}
            \centering
                \includegraphics[width=\linewidth]{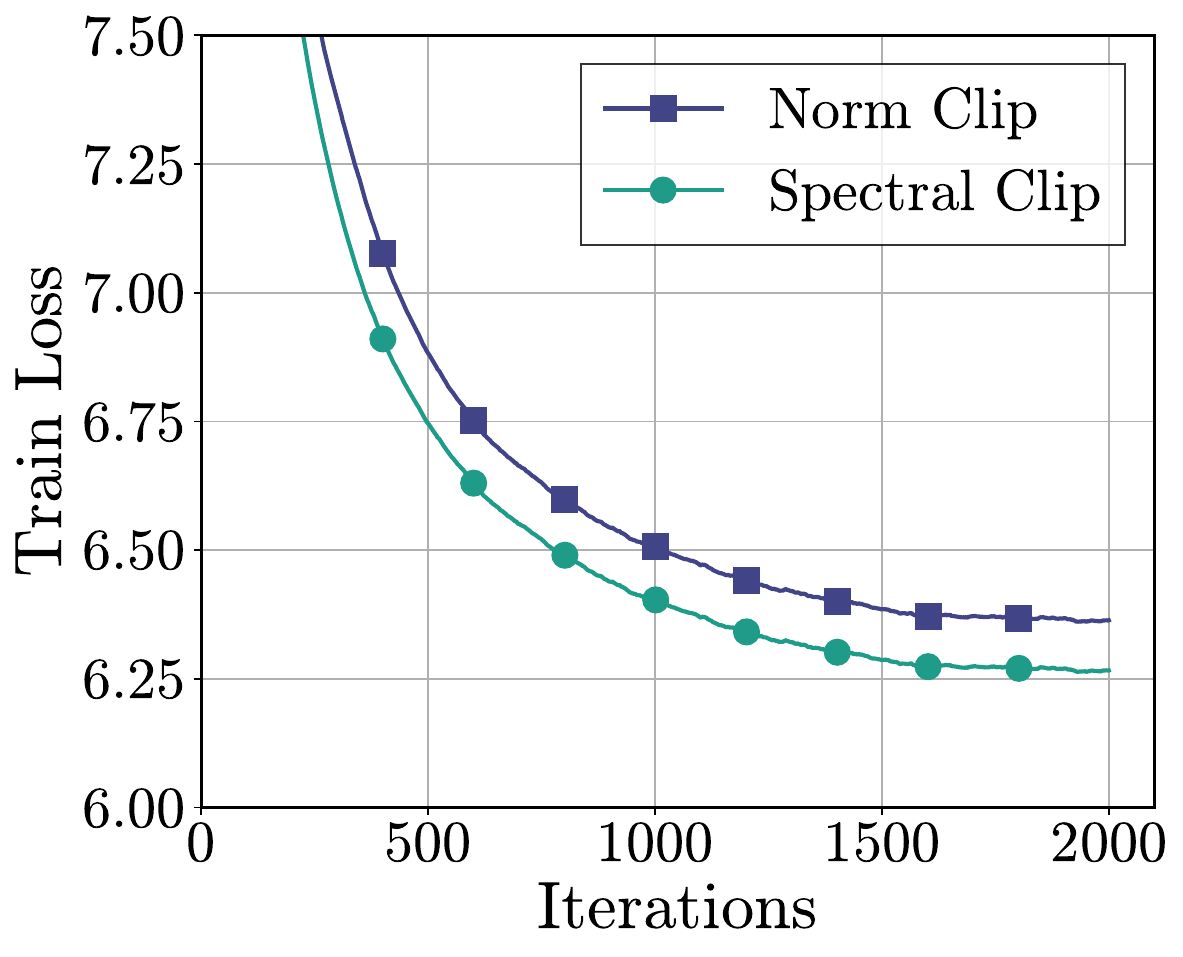}
            \caption{\small \algname{SGDM}}
            \label{fig:gpt2_sgdm}
        \end{subfigure}
        \hfill
        \begin{subfigure}[t]{0.32\linewidth}
            \centering
            \includegraphics[width=\linewidth]{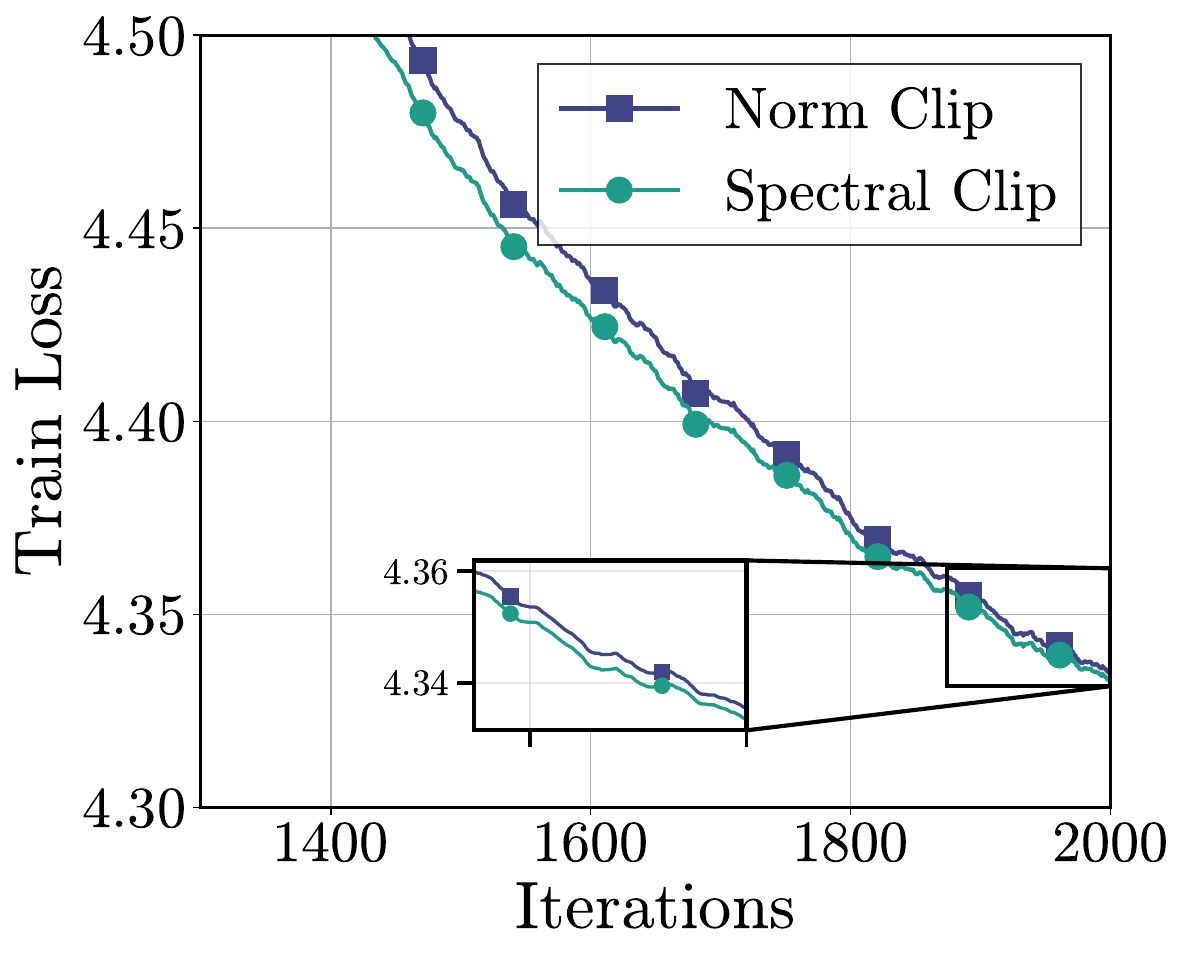}  
            \caption{\small \algname{Muon}}
            \label{fig:gpt2_muon}
        \end{subfigure}
        \hfill
        \begin{subfigure}[t]{0.32\linewidth}
            \centering
            \includegraphics[width=\linewidth]{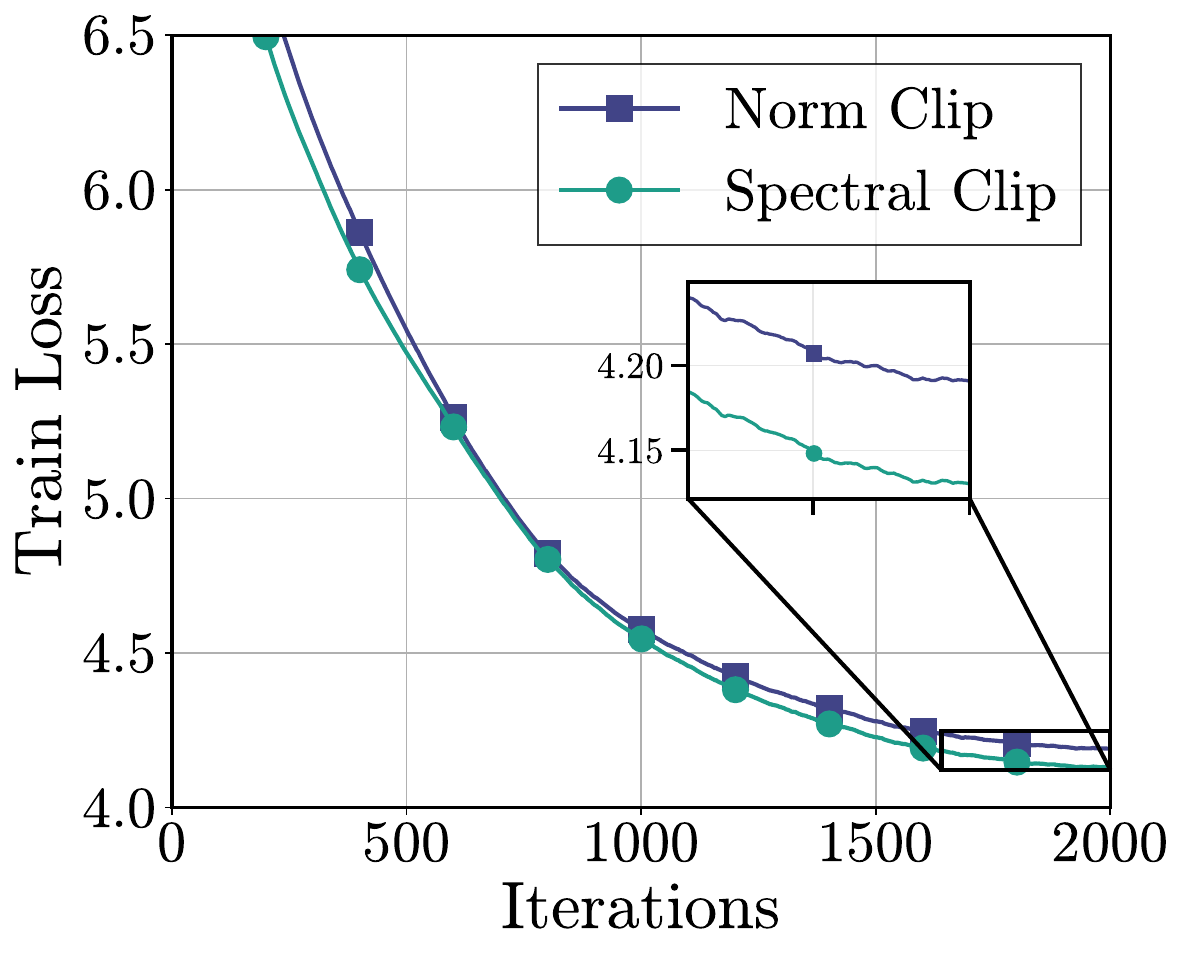}
            \caption{\small \algname{Adam}}
            \label{fig:gpt2_adam}
        \end{subfigure}

      \caption{\small Training loss of GPT-2 on the FineWeb dataset for different optimizers. Left: \algname{SGDM}. Middle: \algname{Muon}. Right: \algname{Adam}. Both the x- and y-axes are truncated to highlight the differences in performance.}
        \label{fig:gpt2}
        \vspace{-.3cm}
    \end{figure}
    \subsection{Training LLM on Shakespeare Dataset}\label{sec:llm}
    \vspace{-.1cm}
    In this section, we evaluated our approach on the NanoGPT~\citep{Karpathy2022} model, trained on the Shakespeare character-level dataset. Our results, shown in Figure~\ref{fig:nanogpt} (Appendix~\ref{sec:missing_figures}), demonstrate that spectral clipping consistently outperforms standard norm-based clipping when used with \algname{SGDM} and \algname{Muon}. 
    These findings suggest that the benefits of spectral truncation extend to language modeling tasks, providing a more stable and efficient optimization trajectory.  \emph{While spectral clipping incurs a higher per-iteration cost due to truncated SVD, its significantly faster convergence allows it to reach a lower loss in less wall-clock time than norm-based clipping as we showed in Figure~\ref{fig:nanogpt_sgdm_wall_time} for \algname{SGDM}}. For instance, after 2 hours of training, \algname{SGDM} with spectral clipping achieves a training loss that is 8.2\% lower than with norm clipping.

    \subsection{Training GPT-2 on FineWeb Dataset}\label{sec:gpt2_fineweb}
    \vspace{-.1cm}
    In this section, we evaluate our approach on the GPT-2 124M
    model~\citep{radford2019language},
    trained from scratch on the FineWeb dataset~\citep{penedo2024fineweb}.
    Our results, shown in Figure~\ref{fig:gpt2}, demonstrate that
    spectral clipping consistently outperforms standard norm-based
    clipping when used with \algname{SGDM}, \algname{Muon}, and
    \algname{Adam}. In Figure~\ref{fig:gpt2_sgdm_wall_time} (Appendix~\ref{sec:missing_figures}), we compared wall-clock time with \algname{SGDM} showing that after 45 minutes of training, \algname{SGDM} with spectral clipping achieves a training loss that is 2.6\% lower than with norm clipping.
    
    \vspace{-.25cm}
    \subsection{Training MLP with Heavy-Tailed Noise}\label{sec:mlp}
    \vspace{-.15cm}

    
    Here, we consider a $4$-layer multilayer perceptron (MLP), where the first three layers have weight matrices of size $100 \times 100$, and the final layer produces a scalar output for regression. The task is a synthetic regression problem: given an input vector $\vx \in \mathbb{R}^{100}$, the network predicts the scalar target $y = x_1 x_2 x_3$, i.e., the product of the first three coordinates of $\vx$.

    To evaluate algorithms under heavy-tailed gradient noise, we add rank-one two-sided Pareto perturbations to gradients with probability $10\%$ at each step, using shape $\alpha=1.1$ and scale $x_m=2.0$ as in~\citet{chezhegov2024clipping}.

    \begin{figure}
        \centering
        \begin{subfigure}[t]{0.32\linewidth}
            \centering
            \includegraphics[width=\linewidth]{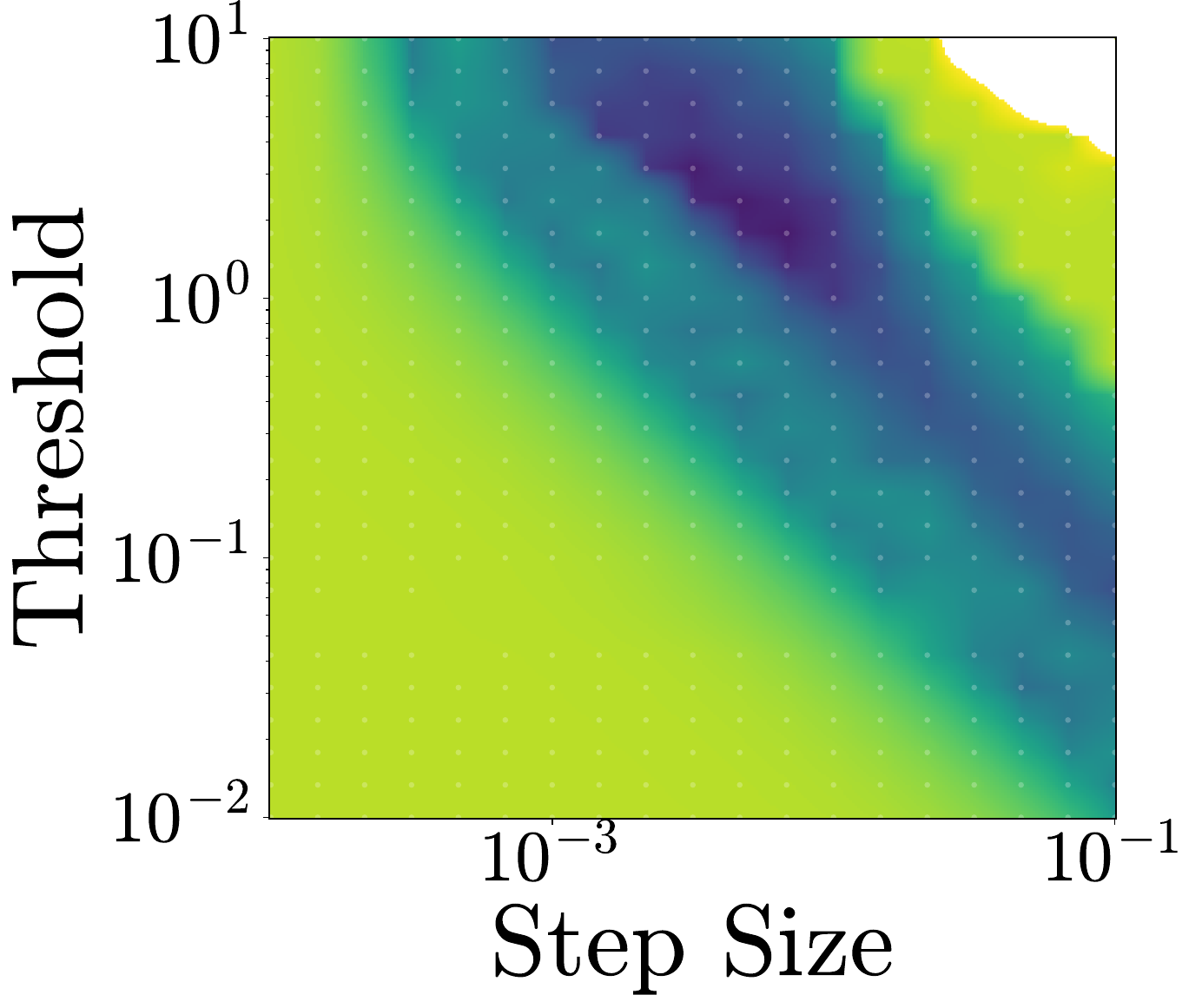}
            \caption{\small Norm clipping}
        \end{subfigure}
        \hfill
        \begin{subfigure}[t]{0.32\linewidth}
            \centering
            \includegraphics[width=\linewidth]{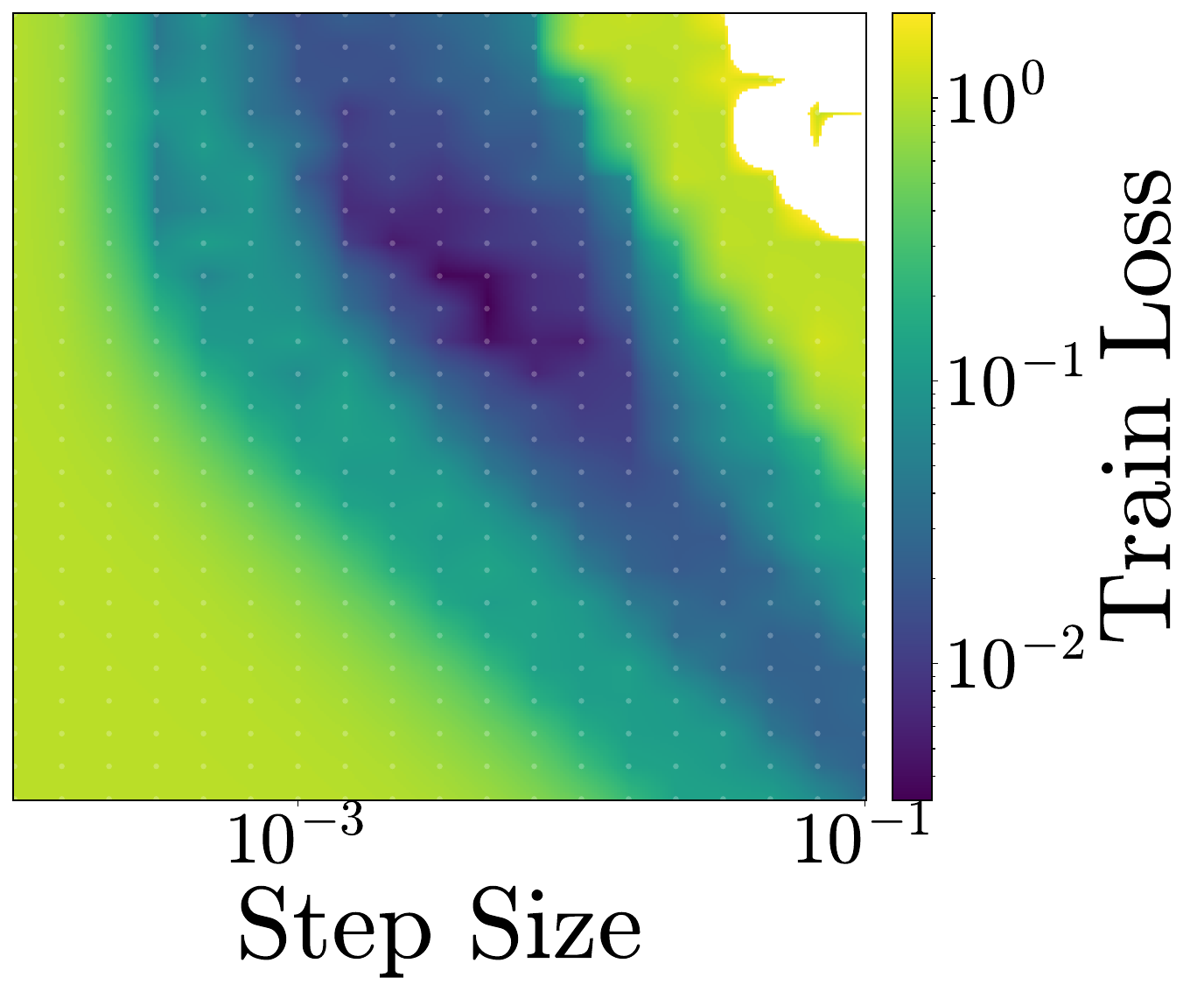}
            \caption{\small Spectral clipping}
            \label{fig:cv_sgdm_curves}
        \end{subfigure}
        \hfill
        \begin{subfigure}[t]{0.32\linewidth}
            \centering
            \includegraphics[width=\linewidth]{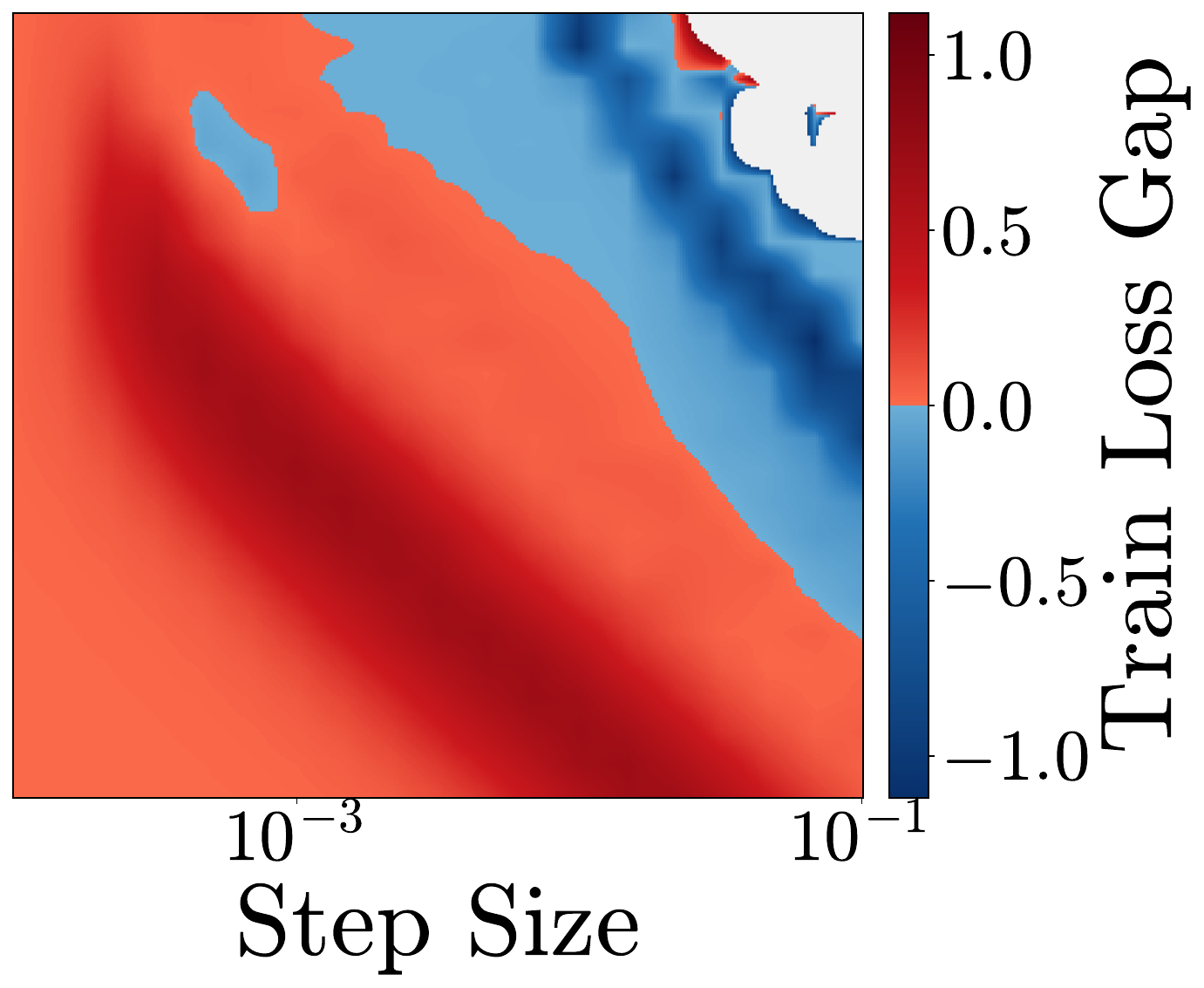}
            \caption{\small Comparison of Training Loss}
            \label{fig:cv_sgdm_curves}
        \end{subfigure}

        \caption{\small Heatmap of train loss in logarithmic scale. Left: norm clipping. Middle: spectral clipping. Right: train loss gap between norm clipping and spectral clipping. Red color shows regions where spectral clipping is better, while blue one shows regions where norm clipping is better.}
        \label{fig:heatmaps}
        \vspace{-.5cm}
    \end{figure}

    Under this setting, we compare \algname{Adam} and \algname{SGDM}. 
    For each optimizer, we evaluate multiple clipping strategies, including no clipping, norm clipping, spectral clipping with a constant threshold, spectral clipping with an EMA-based threshold, and spectral clipping with a quantile-based threshold. 

    As shown in Figure~\ref{fig:mlp_sgdm_30k}, \algname{SGDM} combined with spectral clipping consistently outperforms its norm-clipped counterpart as well as \algname{Adam}. In particular, quantile-based spectral clipping achieves the best performance, while EMA-based spectral clipping performs comparably to carefully tuned constant-threshold spectral and norm clipping. Notably, both quantile-based and EMA-based strategies eliminate the need for manual threshold tuning, \emph{highlighting the practical advantages and effectiveness of our adaptive clipping methods.}

    Additionally, we compare spectral clipping and norm clipping with \algname{SGDM} using a heatmap over threshold and step size. 
    Figure~\ref{fig:heatmaps} shows that spectral clipping achieves lower training loss across a significantly wider range of hyperparameter pairs, which indicates that spectral clipping requires less manual tuning compared to standard gradient clipping.
    


\newpage
\small{\bibliography{references,references_martin}}


\appendix
\newpage

\part*{Supplementary Material}

\tableofcontents

\newpage

\section{Further Related Works}\label{missing_theory}
        \paragraph{Heavy-Tailed Gradient Noise.}
        A growing body of empirical evidence suggests that stochastic gradient noise in modern deep learning is often heavy-tailed~\citep{simsekli2019tail, zhang2020adaptivemethodsgoodattention} rather than in settings captured by the classical bounded variance assumption (Assumption \ref{assume:BV} with $\alpha = 2$). This observation has motivated clipping-based algorithms, adaptive clipping rules, and related normalization analyses under bounded $\alpha$-moment or high-probability assumptions~\citep{gorbunov2020stochastic, zhang2020adaptivemethodsgoodattention, hubler2024gradient, chezhegov2024clipping}. In particular, \citet{zhang2020adaptivemethodsgoodattention} proposed \algname{ACClip}, an adaptive coordinate-wise clipping rule based on moving-average thresholds, and \citet{chezhegov2024clipping} showed that clipping can substantially improve \algname{AdaGrad} and \algname{Adam}-type methods in heavy-tailed regimes. 

        \paragraph{Matrix Structure and Spectral Computations.}
        For a matrix $\mA \in \R^{m \times n}$ with rank $r$, the singular value decomposition (SVD) of $\mA$ is given by $\mU \mathrm{diag}(\bm{\sigma}) \mV^\top$, where $\mU \in \R^{m \times r}$ and $\mV \in \R^{n \times r}$ have orthonormal columns and $\bm{\sigma} \in \R_+^r$ contains the singular values. The SVD plays an important role in many applications~\cite{zhang2015singular}.
        Recent optimizer design has increasingly sought to preserve matrix structure rather than flatten parameters into vectors. Matrix-aware preconditioners such as \algname{Shampoo} and \algname{SOAP} use tensor structure or eigenspace-adapted second moments to construct richer updates~\citep{gupta2018shampoo, vyas2024soap}. Orthogonalization or norm-based methods such as \algname{Muon}, \algname{Scion}, and \algname{ASGO} also operate directly on matrix-valued parameters, but they modify the update direction at every step via orthogonalization, norm-constrained LMOs, or structured preconditioning~\citep{jordan2024muon, pethick2025training, an2025asgo, shah2025practical}.
        The computation of the full decomposition is expensive and requires $\gO(mn^2)$ work. Over the years, researchers have developed several techniques to reduce this cost for specific applications.
        Randomized low-rank decompositions provide a standard route to scalable partial SVD computation~\citep{halko2011finding}, and related low-rank ideas have recently been used in optimization systems such as GaLore~\citep{zhao2024galore}. Orthogonalization-based optimizers such as \algname{Muon} typically avoid exact decompositions via Newton-Schulz \citep{higham1986computing, jordan2024muon}, while PolarExpress~\citep{amsel2025polar} develops faster GPU-friendly polar decomposition routines for this setting.
        In this work, we utilize randomized truncated SVD~\citep{halko2011finding} to clamp the singular values of the gradient matrix beyond a certain threshold.

\section{Technical Lemmas}
    In this section, we present some technical lemmas, which will be used repeatedly to prove the main results of the work in subsequent sections.

    \begin{lemma}[Young's Inequality]
    \label{lemma:young-ineq}
        For any $\mA, \mB \in \R^{m \times n}$ and $\omega > 0$ we have
        \begin{eqnarray}\label{eq:young1}
            \| \mA + \mB \|_F^2 & \leq & (1 + \omega) \| \mA \|_F^2 + (1 + \nicefrac{1}{\omega}) \| \mB\|_F^2.
        \end{eqnarray}
        In particular, for $\omega = 1$, we have
        \begin{eqnarray}\label{eq:young2}
            \| \mA + \mB \|_F^2 & \leq & 2 \| \mA\|_F^2 + 2 \| \mB\|_F^2.
        \end{eqnarray}
    \end{lemma}
    
    \begin{lemma}[Cauchy-Schwarz Inequality]
        For any $\mA, \mB \in \R^{m \times n}$, we have
        \begin{equation}\label{eq:CS1}
            \la \mA, \mB\ra \leq \| \mA\|_F \| \mB\|_F.
        \end{equation}
    \end{lemma}
    
    \begin{lemma}
    \label{lemma:CS2}
        For any $\mA, \mB \in \R^{m \times n}$ and $\omega > 0$, we have
        \begin{eqnarray}\label{eq:CS2}
            \la \mA, \mB \ra & \leq & \frac{\omega}{2} \| \mA \|_F^2 + \frac{1}{2 \omega} \| \mB \|_F^2.
        \end{eqnarray}
    \end{lemma}

    \begin{lemma}[Jensen's Inequality] 
        For a random variable $\rva \in \R$ and a function $g: \R^{m \times n} \to \R$, we have the following:
        \begin{itemize}
            \item if $g$ is convex, then
                \begin{eqnarray}\label{eq:jensen_convex}
                    g \left( \Exp{\rva} \right) & \leq & \Exp{g(\rva)},
                \end{eqnarray}
            \item and if $g$ is concave, then 
                \begin{eqnarray}\label{eq:jensen_concave}
                    g \left( \Exp{\rva} \right) & \geq & \Exp{g(\rva)}.
                \end{eqnarray}
        \end{itemize}
    \end{lemma}

    \begin{lemma}[H\"{o}lder's Inequality] 
        Suppose $p, q \in (1, \infty)$ with $\frac{1}{p} + \frac{1}{q} = 1$ and $\rva, \rvb \in \R$ are random variables. Then we have
        \begin{eqnarray}\label{eq:holder}
            \Exp{| \rva \rvb |} & \leq & \Exp{| \rva |^p}^{\frac{1}{p}} \Exp{| \rvb |^q}^{\frac{1}{q}}.
        \end{eqnarray}
    \end{lemma}

    \begin{lemma}[Markov's Inequality] 
        For the random variable $\rva \in \R$ and a positive number $\varepsilon > 0$, we have
        \begin{eqnarray}\label{eq:markov}
            \Pr[|\rva| \geq \varepsilon] & \leq & \frac{\Exp{|\rva|}}{\varepsilon}.
        \end{eqnarray}
    \end{lemma}


\section{Convergence Analysis}\label{sec:convergence_analysis}
    \begin{lemma}\label{lemma:bound_2_moment} Under Assumption \ref{assume:BV}, we have
        \begin{eqnarray*}
            \Exp{\| \gC_{\tau}(\mG_k)\|_F^2 \mid \mX_k } 
            & \leq & 2 \|\nabla f(\mX_k) \|_F^2 + 2^{4 - \alpha} d^{1 - \nicefrac{\alpha}{2}} \tau^{2-\alpha} \sigma^\alpha .
        \end{eqnarray*}
    \end{lemma}

    \begin{proof}
        \begin{eqnarray*}
            \Exp{\| \gC_{\tau}(\mG_k)\|_F^2 \mid \mX_k } & \leq & 2 \Exp{\| \gC_{\tau}(\mG_k) - \gC_{\tau} (\nabla f(\mX_k))\|_F^2 \mid \mX_k } + 2 \Exp{\| \gC_{\tau}(\nabla f(\mX_k))\|_F^2 \mid \mX_k } \\
            & \leq & 2 \Exp{\| \gC_{\tau}(\mG_k) - \gC_{\tau} (\nabla f(\mX_k))\|_F^2 \mid \mX_k } + 2 \|\nabla f(\mX_k) \|_F^2 \\
            & = & 2 \Exp{\| \gC_{\tau}(\mG_k) - \gC_{\tau} (\nabla f(\mX_k))\|_F^2 \mathbf{1}\left\{ \| \mG_k - \nabla f(\mX_k) \| \geq 2 \sqrt{d} \tau \right\} \mid \mX_k } \\
            && + 2 \Exp{\| \gC_{\tau}(\mG_k) - \gC_{\tau} (\nabla f(\mX_k))\|_F^2 \mathbf{1}\left\{ \| \mG_k - \nabla f(\mX_k) \| < 2 \sqrt{d} \tau \right\} \mid \mX_k } \\
            && + 2 \|\nabla f(\mX_k) \|_F^2 \\
            & \leq & 8 d \tau^2 \Exp{\mathbf{1}\left\{ \| \mG_k - \nabla f(\mX_k) \| \geq 2 \sqrt{d} \tau \right\} \mid \mX_k } \\
            && + 2 \Exp{\| \mG_k - \nabla f(\mX_k)\|_F^2 \mathbf{1}\left\{ \| \mG_k - \nabla f(\mX_k) \| < 2 \sqrt{d} \tau \right\} \mid \mX_k } \\
            && + 2 \|\nabla f(\mX_k) \|_F^2 
        \end{eqnarray*}
        where the last line follows from
        \begin{eqnarray*}
            \| \gC_{\tau}(\mG_k) - \gC_{\tau} (\nabla f(\mX_k))\|_F & \leq & \min \left\{ \| \mG_k - \nabla f(\mX_k) \|_F, 2 \sqrt{d}\tau \right\}.
        \end{eqnarray*}
        Now we bound the first and second terms separately as follows. Note that
        \begin{align*}
            & \Exp{\|\mG_k - \nabla f(\mX_k)\|_F^2 \, \mathbf{1}\{\| \mG_k - \nabla f(\mX_k) \|_F \leq 2 \sqrt{d} \tau \} \mid \mX_k} \\
            \leq & \quad \left( 2 \sqrt{d} \tau \right)^{2-\alpha} \Exp{\| \mG_k - \nabla f(\mX_k) \|_F^\alpha \mid \mX_k} \\
            \leq & \quad \left( 2 \sqrt{d} \tau \right)^{2-\alpha} \sigma^\alpha.
        \end{align*}
        and 
        \begin{eqnarray*}
            \Exp{\mathbf{1}\left\{ \| \mG_k - \nabla f(\mX_k) \| \geq 2 \sqrt{d} \tau \right\} \mid \mX_k } & = & \Pr[\| \mG_k - \nabla f(\mX_k) \| \geq 2 \sqrt{d} \tau \mid \mX_k] \\
            & = & \Pr[\| \mG_k - \nabla f(\mX_k) \|^\alpha \geq \left( 2 \sqrt{d} \tau \right)^\alpha \mid \mX_k] \\
            & \leq & \frac{\Exp{\| \mG_k - \nabla f(\mX_k) \|^\alpha}}{\left( 2 \sqrt{d} \tau \right)^\alpha} \\
            & \leq & \left( 2 \sqrt{d} \tau \right)^{-\alpha} \sigma^\alpha.
        \end{eqnarray*}
        Thus, combining the bounds, we have
        \begin{eqnarray*}
            \Exp{\| \gC_{\tau}(\mG_k)\|_F^2 \mid \mX_k } 
            & \leq & 4 \cdot \left( 2 \sqrt{d} \tau \right)^{2-\alpha} \sigma^\alpha + 2 \|\nabla f(\mX_k) \|_F^2 \\
            & = &  2^{4 - \alpha} d^{1 - \nicefrac{\alpha}{2}} \tau^{2-\alpha} \sigma^\alpha + 2 \|\nabla f(\mX_k) \|_F^2.
        \end{eqnarray*}
        This completes the proof of the lemma.
        
    \end{proof}

\newpage

    \subsection{Proof of Theorem \ref{theorem:SGD_clipping}}\label{sec:proof_theorem:SGD_clipping}
    
    \begin{theorem}
        Suppose Assumptions \ref{assume:smooth} and \ref{assume:BV} hold, and let $f_* \coloneqq \inf_{\mX} f(\mX) > -\infty$. Then for any $K \geq 1$, \algname{SGD} with spectral clipping \eqref{eq:SGD_sclipped} and a constant step size $\eta_k = \eta \eqdef \frac{1}{4 L \sqrt{K}}$ satisfies 
        \begin{eqnarray*}
        \textstyle
            \min_{0 \leq k \leq K-1} \Exp{\| \nabla f(\mX_k) \|_F^2} 
             & \leq & \frac{16 L (f(\mX_0) - f_*)}{\sqrt{K}} + \frac{C_{\alpha, d} \tau^{2 - \alpha} \sigma^\alpha}{8 \sqrt{K}} + \frac{2}{K} \sum_{k = 0}^{K-1} \Delta_k
        \end{eqnarray*}
        where $\Delta_k \eqdef \Exp{\left\| \nabla f(\mX_k) - \Exp {\mathcal{C}_{\tau_k}(\mG_k) \mid \mX_k } \right\|_F^2}$ and $C_{\alpha, d} \eqdef 2^{4 - \alpha} d^{1 - \nicefrac{\alpha}{2}}$ with $d = \min \{ n, m\}$.
    \end{theorem}

    \begin{proof}
        For simplicity, let us denote the clipped gradient estimate by $\mathcal{C}_{\tau}(\mG_k)$. i.e., given $\mG_k = \mU_k \text{diag} (\bm{\sigma}_k) \mV_k^\top$, the clipped gradient with threshold $\tau$ is defined as $\mathcal{C}_{\tau}(\mG_k) \eqdef \mU_k \text{diag}(\clamp_{\tau}(\bm{\sigma}_k)) \mV_k^\top$, where the coordinates of the vector $\clamp_{\tau}(\bm{\sigma})$ are given by:
        \begin{equation*}
            \clamp_{\tau} (\bm{\sigma})_i = \begin{cases}
                \bm{\sigma}_i & \text{if } \bm{\sigma}_i \leq \tau \\
                \tau & \text{otherwise}.
            \end{cases} 
        \end{equation*}
        Using this notation, the update rule can be written as 
        \begin{eqnarray*}
            \mX_{k+1} & = & \mX_k - \eta_k \mathcal{C}_{\tau}(\mG_k).
        \end{eqnarray*}
        Then, from the smoothness property of the function $f$, we have    
        \begin{eqnarray*}
            f(\mX_{k+1}) & \leq & f(\mX_k) + \left \langle \nabla f(\mX_k), \mX_{k+1} - \mX_k \right \rangle + \frac{L}{2} \| \mX_{k+1} - \mX_k \|_F^2 \\
            & = & f(\mX_k) - \eta_k \left \langle \nabla f(\mX_k), \mathcal{C}_{\tau}(\mG_k) \right \rangle + \frac{\eta_k^2 L}{2} \| \mathcal{C}_{\tau}(\mG_k) \|_F^2 \\
            & = & f(\mX_k) - \eta_k \left \langle \nabla f(\mX_k), \mG_k \right \rangle - \eta_k \left \langle \nabla f(\mX_k), \mathcal{C}_{\tau}(\mG_k) - \mG_k \right \rangle  + \frac{\eta_k^2 L}{2} \| \mathcal{C}_{\tau}(\mG_k) \|_F^2.
        \end{eqnarray*}
        Then, taking the expectation conditioned on $\mX_k$, we get
        \begin{eqnarray*}
            \Exp{f(\mX_{k+1}) \mid \mX_k} & \leq & f(\mX_k) - \eta_k \| \nabla f(\mX_k) \|_F^2 + \eta_k \left\langle \nabla f(\mX_k), \Exp{\mG_k - \mathcal{C}_{\tau}(\mG_k) \mid \mX_k} \right\rangle \\
            && + \frac{\eta_k^2 L}{2} \Exp{\| \mathcal{C}_{\tau}(\mG_k)\|_F^2 \mid \mX_k} \\
            & \leq & f(\mX_k) - \eta_k \| \nabla f(\mX_k) \|_F^2 + \frac{\eta_k}{2} \| \nabla f(\mX_k)\|_F^2 + \frac{\eta_k}{2} \left\| \Exp {\mG_k - \mathcal{C}_{\tau}(\mG_k) \mid \mX_k } \right\|_F^2 \\
            && + \frac{\eta_k^2 L}{2} \Exp{\| \mathcal{C}_{\tau}(\mG_k)\|_F^2 \mid \mX_k} \\
            & = & f(\mX_k) - \frac{\eta_k}{2} \| \nabla f(\mX_k)\|_F^2 + \frac{\eta_k}{2} \left\| \nabla f(\mX_k) - \Exp {\mathcal{C}_{\tau}(\mG_k) \mid \mX_k } \right\|_F^2 \\
            && + \frac{\eta_k^2 L}{2} \Exp{\| \mathcal{C}_{\tau}(\mG_k)\|_F^2 \mid \mX_k}. 
        \end{eqnarray*}
        From the Lemma, we have
        \begin{eqnarray*}
            \Exp{\| \gC_{\tau}(\mG_k)\|_F^2 \mid \mX_k } 
            & \leq & 2 \|\nabla f(\mX_k) \|_F^2 + C_{\alpha, d} \tau^{2-\alpha} \sigma^\alpha
        \end{eqnarray*}
        where $C_{\alpha, d} \eqdef 2^{4 - \alpha} d^{1 - \nicefrac{\alpha}{2}}$. Thus, the above inequality provides us
        \begin{eqnarray*}
            \Exp{f(\mX_{k+1}) \mid \mX_k} 
            & \leq & f(\mX_k) - \frac{\eta_k}{2} \| \nabla f(\mX_k)\|_F^2 + \frac{\eta_k}{2} \left\| \nabla f(\mX_k) - \Exp {\mathcal{C}_{\tau}(\mG_k) \mid \mX_k } \right\|_F^2  \\ 
            && + \frac{\eta_k^2 L}{2} \left( 2 \| \nabla f(\mX_k)\|_F^2 + C_{\alpha, d} \tau^{2-\alpha} \sigma^\alpha \right) \\
            & = & f(\mX_k) - \frac{\eta_k}{2} \left( 1 - 2 \eta_k L \right) \| \nabla f(\mX_k) \|_F^2 + \frac{\eta_k}{2} \left\| \nabla f(\mX_k) - \Exp {\mathcal{C}_{\tau}(\mG_k) \mid \mX_k } \right\|_F^2 \\
            && + \frac{\eta_k^2 L C_{\alpha, d} \tau^{2-\alpha} \sigma^\alpha}{2}.
        \end{eqnarray*}
        Now, let us define $\Delta_k \eqdef \Exp{\left\| \nabla f(\mX_k) - \Exp {\mathcal{C}_{\tau}(\mG_k) \mid \mX_k } \right\|_F^2}$. Then, by taking the total expectations on both sides, we obtain
        \begin{eqnarray*}
            \Exp{f(\mX_{k+1})} & \leq & \Exp{f(\mX_k)} - \frac{\eta_k}{2} \left( 1 - 2 \eta_k L \right) \Exp{\| \nabla f(\mX_k) \|_F^2} + \frac{\eta_k}{2} \Delta_k + \frac{\eta_k^2 L C_{\alpha, d} \tau^{2-\alpha} \sigma^\alpha}{2}.
        \end{eqnarray*}
        Then, rearranging the terms, we have
        \begin{eqnarray*}
            \frac{\eta_k}{2} \left( 1 - 2 \eta_k L \right) \Exp{\| \nabla f(\mX_k) \|_F^2} & \leq & \Exp{f(\mX_k)} - \Exp{f(\mX_{k+1})} + \frac{\eta_k}{2} \Delta_k + \frac{\eta_k^2 L C_{\alpha, d} \tau^{2-\alpha} \sigma^\alpha}{2}.
        \end{eqnarray*}
        Now, if we choose $\eta_k \leq \frac{1}{4L}$ for all $k$, then we can lower bound the left hand side with $\frac{\eta_k}{4} \Exp{\| \nabla f(\mX_k) \|_F^2} \leq \frac{\eta_k}{2} \left( 1 - 2 \eta_k L \right) \Exp{\| \nabla f(\mX_k) \|_F^2}$. Hence, we obtain
        \begin{eqnarray*}
            \frac{\eta_k}{4} \Exp{\| \nabla f(\mX_k) \|_F^2} & \leq & \Exp{f(\mX_k)} - \Exp{f(\mX_{k+1})} + \frac{\eta_k}{2} \Delta_k + \frac{\eta_k^2 L C_{\alpha, d} \tau^{2-\alpha} \sigma^\alpha}{2}.
        \end{eqnarray*}
        Now, summing up both sides for $k = 0, 1, \cdots K-1$, we get
        \begin{eqnarray*}
             \sum_{k = 0}^{K-1} \frac{\eta_k}{4} \Exp{\| \nabla f(\mX_k) \|_F^2} & \leq & \Exp{f(\mX_0)} - \Exp{f(\mX_K)} + \frac{1}{2} \sum_{k = 0}^{K-1} \eta_k \Delta_k + \frac{L C_{\alpha, d} \tau^{2-\alpha} \sigma^\alpha}{2} \sum_{k = 0}^{K-1} \eta_k^2 \\
             & \leq & f(\mX_0) - f_* + \frac{1}{2} \sum_{k = 0}^{K-1} \eta_k \Delta_k + \frac{L C_{\alpha, d} \tau^{2-\alpha} \sigma^\alpha}{2} \sum_{k = 0}^{K-1} \eta_k^2.
        \end{eqnarray*}
        The last line follows from $\Exp{f(\mX_K)} \geq f_*$. For constant step size $\eta_k = \eta$, we obtain.
        \begin{eqnarray*}
             \frac{\eta}{4} \sum_{k = 0}^{K-1} \Exp{\| \nabla f(\mX_k) \|_F^2} 
             & \leq & f(\mX_0) - f_* + \frac{\eta}{2} \sum_{k = 0}^{K-1} \Delta_k + \frac{\eta^2 L C_{\alpha, d} \tau^{2-\alpha} \sigma^\alpha K}{2}.
        \end{eqnarray*}
        Now divide both sides by $\frac{\eta K}{4}$ to get
        \begin{eqnarray*}
             \frac{1}{K} \sum_{k = 0}^{K-1} \Exp{\| \nabla f(\mX_k) \|_F^2} 
             & \leq & \frac{4 (f(\mX_0) - f_*)}{\eta K} + \frac{2}{K} \sum_{k = 0}^{K-1} \Delta_k + \frac{\eta L C_{\alpha, d} \tau^{2-\alpha} \sigma^\alpha}{2}.
        \end{eqnarray*}
        Therefore, by lower bounding the left hand side, we obtain 
        \begin{eqnarray*}
             \min_{0 \leq k \leq K-1} \Exp{\| \nabla f(\mX_k) \|_F^2} 
             & \leq & \frac{4 (f(\mX_0) - f_*)}{\eta K} + \frac{2}{K} \sum_{k = 0}^{K-1} \Delta_k + \frac{\eta L C_{\alpha, d} \tau^{2-\alpha} \sigma^\alpha}{2}.
        \end{eqnarray*}
        For $\eta = \frac{1}{4 L \sqrt{K}}$, we obtain
        \begin{eqnarray*}
             \min_{0 \leq k \leq K-1} \Exp{\| \nabla f(\mX_k) \|_F^2} 
             & \leq & \frac{16 L (f(\mX_0) - f_*)}{\sqrt{K}} + \frac{2}{K} \sum_{k = 0}^{K-1} \Delta_k + \frac{ C_{\alpha, d} \tau^{2-\alpha} \sigma^\alpha}{8 \sqrt{K}}.
        \end{eqnarray*}
        This completes the proof of the theorem.
    \end{proof}

    \subsection{Proof of Theorem \ref{theorem:main_theorem}}

    \begin{theorem}
        Suppose Assumptions \ref{assume:smooth} and \ref{assume:BV} hold, and let $f_* \coloneqq \inf_{\mX} f(\mX) > -\infty$. Then, for any $K \geq 1$, \algname{SGD} with spectral clipping \eqref{eq:SGD_sclipped}, a constant step size $\eta_k = \eta \leq \nicefrac{1}{4dL}$ and clipping threshold $\tau \geq \max \{ 2, 8 \sigma \}$ satisfies the following
        \begin{align*}
            \min_{0 \leq k \leq K-1} \Exp{ \min \left\{ \| \nabla f(\mX_k)\|_F, \| \nabla f(\mX_k)\|_F^2 \right\}}  
            & \leq \frac{4 (f(\mX_0) - f_*)}{\eta K} + 32 \tau^{-2(\alpha - 1)} \sigma^{2 \alpha} + 8 \eta d L \tau^{2 - \alpha} \sigma^{\alpha}. 
        \end{align*}
    \end{theorem}

    \begin{proof}
        For simplicity, let us denote the clipped gradient estimate with $\mathcal{C}_{\tau}(\mG_k)$. i.e., given $\mG_k = \mU_k \text{diag} (\bm{\sigma}_k) \mV_k^\top$, the clipped gradient with threshold $\tau$ is defined as $\mathcal{C}_{\tau}(\mG_k) \eqdef \mU_k \text{diag}(\clamp_{\tau}(\bm{\sigma}_k)) \mV_k^\top$, where the coordinates of the vector $\clamp_{\tau}(\bm{\sigma})$ are given by:
        \begin{equation*}
            \clamp_{\tau} (\bm{\sigma})_i = \begin{cases}
                \bm{\sigma}_i & \text{if } \bm{\sigma}_i \leq \tau \\
                \tau & \text{otherwise}.
            \end{cases} 
        \end{equation*}
        Using this notation, the update rule can be written as 
        \begin{eqnarray*}
            \mX_{k+1} & = & \mX_k - \eta_k \mathcal{C}_{\tau}(\mG_k).
        \end{eqnarray*}
        Then, from the smoothness property of the function $f$, we have    
        \begin{eqnarray*}
            f(\mX_{k+1}) & \leq & f(\mX_k) + \left \langle \nabla f(\mX_k), \mX_{k+1} - \mX_k \right \rangle + \frac{L}{2} \| \mX_{k+1} - \mX_k \|_F^2 \\
            & = & f(\mX_k) - \eta_k \left \langle \nabla f(\mX_k), \mathcal{C}_{\tau}(\mG_k) \right \rangle + \frac{\eta_k^2 L}{2} \| \mathcal{C}_{\tau}(\mG_k) \|_F^2 \\
            & = & f(\mX_k) - \eta_k \left \langle \nabla f(\mX_k), \mG_k \right \rangle - \eta_k \left \langle \nabla f(\mX_k), \mathcal{C}_{\tau}(\mG_k) - \mG_k \right \rangle  + \frac{\eta_k^2 L}{2} \| \mathcal{C}_{\tau}(\mG_k) \|_F^2.
        \end{eqnarray*}
        Then, taking the expectation conditioned on $\mX_k$, we obtain
        \begin{eqnarray*}
            \Exp{f(\mX_{k+1}) \mid \mX_k} & \leq & f(\mX_k) - \eta_k \| \nabla f(\mX_k) \|_F^2 + \eta_k \left\langle \nabla f(\mX_k), \Exp{\mG_k - \mathcal{C}_{\tau}(\mG_k) \mid \mX_k} \right\rangle \\
            && + \frac{\eta_k^2 L}{2} \Exp{\| \mathcal{C}_{\tau}(\mG_k)\|_F^2 \mid \mX_k} \\
            & \leq & f(\mX_k) - \eta_k \| \nabla f(\mX_k) \|_F^2 + \frac{\eta_k}{2} \| \nabla f(\mX_k)\|_F^2 + \frac{\eta_k}{2} \left\| \Exp {\mG_k - \mathcal{C}_{\tau}(\mG_k) \mid \mX_k } \right\|_F^2 \\
            && + \frac{\eta_k^2 L}{2} \Exp{\| \mathcal{C}_{\tau}(\mG_k)\|_F^2 \mid \mX_k} \\
            & = & f(\mX_k) - \frac{\eta_k}{2} \| \nabla f(\mX_k)\|_F^2 + \frac{\eta_k}{2} \left\| \nabla f(\mX_k) - \Exp {\mathcal{C}_{\tau}(\mG_k) \mid \mX_k } \right\|_F^2 \\
            && + \frac{\eta_k^2 L}{2} \Exp{\| \mathcal{C}_{\tau}(\mG_k)\|_F^2 \mid \mX_k}. 
        \end{eqnarray*}

        For further analysis, we will consider two cases: either $\| \nabla f(\mX_k )\|_F \leq \nicefrac{\tau}{2}$ or $\| \nabla f(\mX_k) \|_F > \nicefrac{\tau}{2}$. 

        \textbf{Case 1.} Suppose $\| \nabla f(\mX_k)\|_F \leq \nicefrac{\tau}{2}$. Then we have the following bound
        \begin{align*}
            \| \gC_{\tau} (\mG_k)\|_F^2 & = \| \gC_{\tau} (\mG_k)\|_F^2 \mathbf{1}\left\{ \|\mG_k \|_F \leq \tau \right\} + \| \gC_{\tau} (\mG_k)\|_F^2 \mathbf{1}\left\{ \|\mG_k \|_F > \tau \right\} \\
            & \leq \| \mG_k \|_F^2 \mathbf{1}\left\{ \|\mG_k \|_F \leq \tau \right\} + d \tau^2 \mathbf{1}\left\{ \|\mG_k \|_F > \tau \right\} \\
            & \leq 2\| \nabla f(\mX_k)\|_F^2 + 2 \| \mG_k - \nabla f(\mX_k) \|_F^2 \mathbf{1}\left\{ \|\mG_k \|_F \leq \tau \right\} + d \tau^2 \mathbf{1}\left\{ \|\mG_k \|_F > \tau \right\} \\
            & \leq 2\| \nabla f(\mX_k)\|_F^2 + 2 \| \mG_k - \nabla f(\mX_k) \|_F^2 \mathbf{1}\left\{ \|\mG_k - \nabla f(\mX_k) \|_F \leq \nicefrac{3 \tau}{2} \right\} + d \tau^2 \mathbf{1}\left\{ \|\mG_k \|_F > \tau \right\} \\
            & \leq 2\| \nabla f(\mX_k)\|_F^2 + 2 (\nicefrac{3\tau}{2})^{2 - \alpha} \| \mG_k - \nabla f(\mX_k) \|_F^\alpha + d \tau^2 \mathbf{1}\left\{ \|\mG_k \|_F > \tau \right\} 
        \end{align*}
        Then, taking the expectation, we have
        \begin{eqnarray*}
            \Exp{\| \gC_{\tau} (\mG_k)\|_F^2 \mid \mX_k}
            & \leq & 2\| \nabla f(\mX_k)\|_F^2 + 2 (\nicefrac{3\tau}{2})^{2 - \alpha} \sigma^\alpha + d \tau^2 \Exp{\mathbf{1}\left\{ \|\mG_k \|_F > \tau \right\} \mid \mX_k} 
        \end{eqnarray*}
        
        Moreover, for $\| \nabla f(\mX_k)\|_F \leq \nicefrac{\tau}{2}$, we also have $\mathbf{1}\left\{ \| \mG_k \|_F > \tau \right\} \leq \mathbf{1}\left\{ \| \mG_k - \nabla f(\mX_k) \|_F > \nicefrac{\tau}{2} \right\}$. Thus, using Markov's Inequality, we obtain
        \begin{eqnarray*}
            \Exp{\mathbf{1}\left\{ \| \mG_k \|_F > \tau \right\} \mid \mX_k} & \leq & \Exp{\mathbf{1}\left\{ \| \mG_k - \nabla f(\mX_k) \|_F > \nicefrac{\tau}{2} \right\} \mid \mX_k} \\
            & = & \Pr[\| \mG_k - \nabla f(\mX_k) \|_F > \nicefrac{\tau}{2} \mid \mX_k] \\
            & = & \Pr[\| \mG_k - \nabla f(\mX_k) \|_F^\alpha > (\nicefrac{\tau}{2})^\alpha \mid \mX_k ] \\
            & \leq & \frac{\Exp{\| \mG_k - \nabla f(\mX_k) \|_F^\alpha \mid \mX_k}}{(\nicefrac{\tau}{2})^\alpha} \\
            & \leq & 2^\alpha \tau^{- \alpha} \sigma^\alpha.
        \end{eqnarray*}
        Thus, combining the bounds, we have shown
        \begin{eqnarray*}
            \Exp{\| \gC_{\tau} (\mG_k)\|_F^2 \mid \mX_k}
            & \leq & 2\| \nabla f(\mX_k)\|_F^2 + C_{\alpha, 1} \tau^{2 - \alpha} \sigma^{\alpha}
        \end{eqnarray*}
        where $C_{\alpha, 1} \eqdef 2^\alpha d + 3 \cdot (2/3)^{\alpha - 1} \leq 8d$. Similarly, we want to bound $\left\| \nabla f(\mX_k) - \Exp {\mathcal{C}_{\tau}(\mG_k) \mid \mX_k } \right\|_F^2$. Note that
        \begin{eqnarray*}
            \left\| \nabla f(\mX_k) - \Exp {\mathcal{C}_{\tau}(\mG_k) \mid \mX_k } \right\|_F^2 & = & \left\| \Exp {\mG_k - \mathcal{C}_{\tau}(\mG_k) \mid \mX_k } \right\|_F^2 \\
            & \leq & \Exp{\left\| \mG_k - \mathcal{C}_{\tau}(\mG_k) \right\|_F \mid \mX_k}^2 \\
            & = & \left( \Exp{\left\| \mG_k - \mathcal{C}_{\tau}(\mG_k) \right\|_F \mathbf{1}\left\{ \| \mG_k\|_F \leq \tau \right\} \mid \mX_k} \right. \\
            && + \left. \Exp{\left\| \mG_k - \mathcal{C}_{\tau}(\mG_k) \right\|_F \mathbf{1}\left\{ \| \mG_k\|_F > \tau \right\} \mid \mX_k} \right)^2 \\
            & = & \Exp{\left\| \mG_k - \mathcal{C}_{\tau}(\mG_k) \right\|_F \mathbf{1}\left\{ \| \mG_k\|_F > \tau \right\} \mid \mX_k}^2 \\
            & \leq & \Exp{\left\| \mG_k \right\|_F \mathbf{1}\left\{ \| \mG_k\|_F > \tau \right\} \mid \mX_k}^2 \\
            & \leq & \Exp{\left\| \mG_k \right\|_F \mathbf{1}\left\{ \| \mG_k - \nabla f(\mX_k) \|_F > \nicefrac{\tau}{2} \right\} \mid \mX_k}^2
        \end{eqnarray*}
        where the last inequality follows from the observation that, for $\| \nabla f(\mX_k)\|_F \leq \nicefrac{\tau}{2}$, we have $\mathbf{1}\left\{ \| \mG_k \|_F > \tau \right\} \leq \mathbf{1}\left\{ \| \mG_k - \nabla f(\mX_k) \|_F > \nicefrac{\tau}{2} \right\}$. Then
        \begin{align*}
            & \quad \left\| \nabla f(\mX_k) - \Exp {\mathcal{C}_{\tau}(\mG_k) \mid \mX_k } \right\|_F^2 \\
            \leq & \quad \Exp{\left\| \mG_k \right\|_F \mathbf{1}\left\{ \| \mG_k - \nabla f(\mX_k) \|_F > \nicefrac{\tau}{2} \right\} \mid \mX_k}^2 \\
            \leq & \quad \Exp{ \left(\left\| \mG_k - \nabla f(\mX_k)\right\|_F + \| \nabla f(\mX_k)\|_F \right) \mathbf{1}\left\{ \| \mG_k - \nabla f(\mX_k) \|_F > \nicefrac{\tau}{2} \right\} \mid \mX_k}^2 \\
            \leq & \quad 4 \Exp{\left\| \mG_k - \nabla f(\mX_k)\right\|_F \mathbf{1}\left\{ \| \mG_k - \nabla f(\mX_k) \|_F > \nicefrac{\tau}{2} \right\} \mid \mX_k}^2
        \end{align*}
        where the last line follows from $\| \nabla f(\mX_k)\|_F \leq \nicefrac{\tau}{2}$ Then we apply Holder's Inequality to get 
        \begin{align*}
            & \quad \Exp{ \left\| \mG_k - \nabla f(\mX_k)\right\|_F \mathbf{1}\left\{ \| \mG_k - \nabla f(\mX_k) \|_F > \nicefrac{\tau}{2} \right\} \mid \mX_k} \\
            \leq & \quad \left( \Exp{ \| \mG_k - \nabla f(\mX_k)\|_F^\alpha \mid \mX_k}\right)^{\nicefrac{1}{\alpha}} \cdot \left( \Pr[\| \mG_k - \nabla f(\mX_k)\|_F > \nicefrac{\tau}{2} \mid \mX_k]\right)^{1 - \nicefrac{1}{\alpha}} \\
            \leq & \quad \left(\sigma^\alpha \right)^{\nicefrac{1}{\alpha}}\cdot \left(\frac{\Exp{ \| \mG_k - \nabla f(\mX_k) \|_F^\alpha \mid \mX_k}}{(\nicefrac{\tau}{2})^\alpha} \right)^{\nicefrac{\alpha - 1}{\alpha}} \\
            = & \quad \sigma \cdot \left( \frac{2 \sigma}{\tau} \right)^{\alpha - 1}.
        \end{align*}
        Thus, by combining the bounds, we obtain.
        \begin{eqnarray*}
            \left\| \nabla f(\mX_k) - \Exp {\mathcal{C}_{\tau}(\mG_k) \mid \mX_k } \right\|_F^2 & \leq & C_{\alpha, 2} \tau^{-2(\alpha - 1)} \sigma^{2 \alpha}
        \end{eqnarray*}
        where $C_{\alpha, 2} \eqdef 2^{2\alpha} \leq 16$. Hence, for the $\| \nabla f(\mX_k)\|_F \leq \nicefrac{\tau}{2}$ case, we have
        \begin{eqnarray*}
            \Exp{f(\mX_{k+1}) \mid \mX_k} & \leq & f(\mX_k) - \frac{\eta_k}{2} \left(1 - 2\eta_k L \right) \| \nabla f(\mX_k)\|_F^2 + \frac{\eta_k}{2} C_{\alpha, 2} \tau^{-2(\alpha - 1)} \sigma^{2 \alpha} \\ 
            && + \frac{\eta_k^2 L}{2} C_{\alpha, 1} \tau^{2 - \alpha} \sigma^{\alpha}. 
        \end{eqnarray*}     
        Then for $\eta_k = \eta \leq \frac{1}{4L}$, we obtain
        \begin{eqnarray*}
            \Exp{f(\mX_{k+1}) \mid \mX_k} & \leq & f(\mX_k) - \frac{\eta_k}{4} \| \nabla f(\mX_k)\|_F^2 + \frac{\eta_k}{2} C_{\alpha, 2} \tau^{-2(\alpha - 1)} \sigma^{2 \alpha} + \frac{\eta_k^2 L}{2} C_{\alpha, 1} \tau^{2 - \alpha} \sigma^{\alpha} \\
             & \leq & f(\mX_k) - \frac{\eta_k}{4} \min \left\{ \| \nabla f(\mX_k)\|_F, \| \nabla f(\mX_k)\|_F^2 \right\} + \frac{\eta_k}{2} C_{\alpha, 2} \tau^{-2(\alpha - 1)} \sigma^{2 \alpha} \\
             && + \frac{\eta_k^2 L}{2} C_{\alpha, 1} \tau^{2 - \alpha} \sigma^{\alpha}. 
        \end{eqnarray*}  

        \textbf{Case 2.} Now, consider $\| \nabla f(\mX_k)\|_F > \nicefrac{\tau}{2}$. Similar to the previous case, we also have the following bound
        \begin{eqnarray*}
            f(\mX_{k+1}) & \leq & f(\mX_k) + \left \langle \nabla f(\mX_k), \mX_{k+1} - \mX_k \right \rangle + \frac{L}{2} \| \mX_{k+1} - \mX_k \|_F^2 \\
            & = & f(\mX_k) - \eta_k \left \langle \nabla f(\mX_k), \mathcal{C}_{\tau}(\mG_k) \right \rangle + \frac{\eta_k^2 L}{2} \| \mathcal{C}_{\tau}(\mG_k) \|_F^2 \\
            & \leq & f(\mX_k) - \eta_k \left \langle \nabla f(\mX_k), \mathcal{C}_{\tau}(\mG_k) \right \rangle + \frac{\eta_k^2 d L \tau^2}{2} 
        \end{eqnarray*}
        Here, we take the expectation conditioned on $\mX_k$ to obtain
        \begin{eqnarray*}
            \Exp{f(\mX_{k+1}) \mid \mX_k} & \leq & f(\mX_k) - \eta_k \langle \nabla f(\mX_k), \Exp{\gC_{\tau}(\mG_k) \mid \mX_k} \rangle + \frac{\eta_k^2 d L \tau^2}{2} \\
            & = & f(\mX_k) - \eta_k \langle \nabla f(\mX_k), \gC_{\tau}(\nabla f(\mX_k)) \rangle \\
            && - \eta_k \langle \nabla f(\mX_k), \Exp{\gC_{\tau}(\mG_k) - \gC_{\tau}(\nabla f(\mX_k)) \mid \mX_k} \rangle + \frac{\eta_k^2 d L \tau^2}{2} \\
            & = & f(\mX_k) - \eta_k \langle \nabla f(\mX_k), \gC_{\tau}(\nabla f(\mX_k)) \rangle \\
            && + \eta_k \| \nabla f(\mX_k)\|_F \| \Exp{\gC_{\tau}(\mG_k) - \gC_{\tau}(\nabla f(\mX_k)) \mid \mX_k} \|_F + \frac{\eta_k^2 d L \tau^2}{2} \\
            & \leq & f(\mX_k) - \eta_k \langle \nabla f(\mX_k), \gC_{\tau}(\nabla f(\mX_k)) \rangle \\
            && + \eta_k \| \nabla f(\mX_k)\|_F \Exp{ \| \gC_{\tau}(\mG_k) - \gC_{\tau}(\nabla f(\mX_k)) \|_F \mid \mX_k} + \frac{\eta_k^2 d L \tau^2}{2} \\
            & \leq & f(\mX_k) - \eta_k \langle \nabla f(\mX_k), \gC_{\tau}(\nabla f(\mX_k)) \rangle \\
            && + \eta_k \| \nabla f(\mX_k)\|_F \Exp{ \| \mG_k - \nabla f(\mX_k) \|_F \mid \mX_k} + \frac{\eta_k^2 d L \tau^2}{2} \\
            & \leq & f(\mX_k) - \eta_k \langle \nabla f(\mX_k), \gC_{\tau}(\nabla f(\mX_k)) \rangle \\
            && + \eta_k \| \nabla f(\mX_k)\|_F \Exp{ \| \mG_k - \nabla f(\mX_k) \|_F^\alpha \mid \mX_k}^{\nicefrac{1}{\alpha}} + \frac{\eta_k^2 d L \tau^2}{2} \\
            & \leq & f(\mX_k) - \eta_k \langle \nabla f(\mX_k), \gC_{\tau}(\nabla f(\mX_k)) \rangle + \eta_k \sigma \| \nabla f(\mX_k)\|_F + \frac{\eta_k^2 d L \tau^2}{2}.
        \end{eqnarray*}
        where the second-to-last inequality follows from Jensen's Inequality with $\alpha \in (1, 2]$. Now we want to lower bound the term $\langle \nabla f(\mX_k), \gC_{\tau}(\nabla f(\mX_k)) \rangle$ when $\| \nabla f(\mX_k)\|_F > \nicefrac{\tau}{2}$. Note that
        \begin{eqnarray*}
            \langle \nabla f(\mX_k), \gC_{\tau}(\nabla f(\mX_k)) \rangle & = & \sum_{i = 1}^d \sigma_i \cdot \min \left\{ \sigma_i, \tau \right\}
        \end{eqnarray*}
        where $\{ \sigma_i \}_{i = 1}^d$ denotes the singular values of $\nabla f(\mX_k)$. Now we consider two cases:
        
        \textbf{Case 1.} $\nicefrac{\tau}{2} < \| \nabla f(\mX_k)\|_F \leq \tau$. In this case, as $\| \nabla f(\mX_k)\|_F \leq \tau$, all the singular values of $\nabla f(\mX_k)$ are less than $\tau$. Therefore 
        \begin{eqnarray*}
            \langle \nabla f(\mX_k), \gC_{\tau}(\nabla f(\mX_k)) \rangle 
            & = & \sum_{i = 1}^d \sigma_i^2 \\
            & = & \| \nabla f(\mX_k)\|_F^2 \\
            & \geq & \frac{\tau}{2} \| \nabla f(\mX_k) \|_F.
        \end{eqnarray*}
        \textbf{Case 2.} $\| \nabla f(\mX_k)\|_F > \tau$. In this case, For each singular value $\sigma_i$, we have
        \begin{eqnarray*}
            \min \left\{ \sigma_i, \tau \right\} & \geq & \frac{\sigma_i \tau}{\| \nabla f(\mX_k)\|_F}.
        \end{eqnarray*}
        Therefore, 
        \begin{eqnarray*}
            \langle \nabla f(\mX_k), \gC_{\tau}(\nabla f(\mX_k)) \rangle & = & \sum_{i = 1}^d \sigma_i \cdot \min \left\{ \sigma_i, \tau \right\} \\
            & \geq & \frac{\tau}{\| \nabla f(\mX_k)\|_F} \sum_{i = 1}^d \sigma_i^2 \\
            & = & \tau \| \nabla f(\mX_k)\|_F \\
            & \geq & \frac{\tau}{2} \| \nabla f(\mX_k)\|_F.
        \end{eqnarray*}
        Therefore, we have shown $\langle \nabla f(\mX_k), \gC_{\tau}(\nabla f(\mX_k)) \rangle \geq \frac{\tau}{2} \| \nabla f(\mX_k)\|_F$ when $ \| \nabla f(\mX_k)\|_F > \nicefrac{\tau}{2}$. Finally, combining the bounds, we get
        \begin{eqnarray*}
            \Exp{f(\mX_{k+1}) \mid \mX_k} 
            & \leq & f(\mX_k) - \frac{\eta_k \tau}{2} \| \nabla f(\mX_k)\|_F + \eta_k \sigma \| \nabla f(\mX_k)\|_F + \frac{\eta_k^2 d L \tau^2}{2} \\
            & \leq & f(\mX_k) - \frac{\eta_k \tau}{2} \| \nabla f(\mX_k)\|_F + \eta_k \sigma \| \nabla f(\mX_k)\|_F + \eta_k^2 d L \tau \| \nabla f(\mX_k) \|_F \\
            & \leq & f(\mX_k) - \frac{\eta_k}{2} \left( \tau (1 - 2 \eta_k dL) - 2\sigma \right) \| \nabla f(\mX_k)\|_F.
        \end{eqnarray*}
        Here, we choose $\eta_k \leq \nicefrac{1}{4 d L}$ to obtain
        \begin{eqnarray*}
            \Exp{f(\mX_{k+1}) \mid \mX_k} 
            & \leq & f(\mX_k) - \frac{\eta_k}{2} \left( \frac{\tau}{2} - 2\sigma \right) \| \nabla f(\mX_k)\|_F.
        \end{eqnarray*}
        For $\tau \geq 8 \sigma$, we can further bound it as shown below
        \begin{eqnarray*}
            \Exp{f(\mX_{k+1}) \mid \mX_k} 
            & \leq & f(\mX_k) - \frac{\eta_k \tau}{8} \| \nabla f(\mX_k)\|_F \\
            & \leq & f(\mX_k) - \frac{\eta_k \tau}{8} \| \nabla f(\mX_k)\|_F + \frac{\eta_k}{2} C_{\alpha, 2} \tau^{-2(\alpha - 1)} \sigma^{2 \alpha} + \frac{\eta_k^2 L}{2} C_{\alpha, 1} \tau^{2 - \alpha} \sigma^{\alpha}. 
        \end{eqnarray*}
        For $\tau \geq 2$, we have
        \begin{eqnarray*}
            \Exp{f(\mX_{k+1}) \mid \mX_k} 
            & \leq & f(\mX_k) - \frac{\eta_k}{4} \| \nabla f(\mX_k)\|_F + \frac{\eta_k}{2} C_{\alpha, 2} \tau^{-2(\alpha - 1)} \sigma^{2 \alpha} + \frac{\eta_k^2 L}{2} C_{\alpha, 1} \tau^{2 - \alpha} \sigma^{\alpha} \\
            & \leq & f(\mX_k) - \frac{\eta_k}{4} \min \left\{ \| \nabla f(\mX_k)\|_F, \| \nabla f(\mX_k)\|_F^2 \right\} + \frac{\eta_k}{2} C_{\alpha, 2} \tau^{-2(\alpha - 1)} \sigma^{2 \alpha} \\
            && \qquad + \frac{\eta_k^2 L}{2} C_{\alpha, 1} \tau^{2 - \alpha} \sigma^{\alpha}. 
        \end{eqnarray*}

        Finally, taking the total expectation, rearranging, summing up the inequality for $k = 0, 1, \cdots, K-1$, and dividing by $K$, we have
        \begin{eqnarray*}
            \frac{\eta}{4K} \sum_{k = 0}^{K-1} \Exp{\min \left\{ \| \nabla f(\mX_k)\|_F, \| \nabla f(\mX_k)\|_F^2 \right\}}  & \leq & \frac{f(\mX_0) - f_*}{K} + \frac{\eta}{2} C_{\alpha, 2} \tau^{-2(\alpha - 1)} \sigma^{2 \alpha} \\
            && \qquad + \frac{\eta^2 L}{2} C_{\alpha, 1} \tau^{2 - \alpha} \sigma^{\alpha}. 
        \end{eqnarray*}
        Dividing both sides by $\nicefrac{\eta}{4}$, we obtain
        \begin{eqnarray*}
            \frac{1}{K} \sum_{k = 0}^{K-1} \Exp{\min \left\{ \| \nabla f(\mX_k)\|_F, \| \nabla f(\mX_k)\|_F^2 \right\} }& \leq & \frac{4 (f(\mX_0) - f_*)}{\eta K} + 2 C_{\alpha, 2} \tau^{-2(\alpha - 1)} \sigma^{2 \alpha} \\
            && \qquad + \eta L C_{\alpha, 1} \tau^{2 - \alpha} \sigma^{\alpha}. 
        \end{eqnarray*}
        This implies
        \begin{eqnarray*}
            \min_{0 \leq k \leq K-1} \Exp{ \min \left\{ \| \nabla f(\mX_k)\|_F, \| \nabla f(\mX_k)\|_F^2 \right\}}  & \leq & \frac{4 (f(\mX_0) - f_*)}{\eta K} + 2 C_{\alpha, 2} \tau^{-2(\alpha - 1)} \sigma^{2 \alpha} \\
            && \qquad + \eta L C_{\alpha, 1} \tau^{2 - \alpha} \sigma^{\alpha} \\
            & \leq & \frac{4 (f(\mX_0) - f_*)}{\eta K} + 32 \tau^{-2(\alpha - 1)} \sigma^{2 \alpha} \\
            && \qquad + 8 \eta d L \tau^{2 - \alpha} \sigma^{\alpha}. 
        \end{eqnarray*}
        This completes the proof of the first part of the Theorem. Now, for the second part, consider $\eta \leq \nicefrac{1}{L\tau^\alpha}$, which implies
        \begin{eqnarray*}
            32 \tau^{-2(\alpha - 1)} \sigma^{2 \alpha} + 8 \eta d L \tau^{2 - \alpha} \sigma^{\alpha} & \leq & 32 d \frac{\sigma^{2 \alpha} + \sigma^\alpha}{\tau^{2 \alpha - 2}}.
        \end{eqnarray*}
        Thus, the above bound reduces to 
        \begin{eqnarray*}
            \min_{0 \leq k \leq K-1} \Exp{ \min \left\{ \| \nabla f(\mX_k)\|_F, \| \nabla f(\mX_k)\|_F^2 \right\}}  & \leq & \frac{4 (f(\mX_0) - f_*)}{\eta K} + 32 d \frac{\sigma^{2 \alpha} + \sigma^\alpha}{\tau^{2 \alpha - 2}}.
        \end{eqnarray*}
        Here, we choose $\tau \geq \sigma K^{\frac{1}{3 \alpha - 2}}$ to obtain
        \begin{eqnarray*}
            \min_{0 \leq k \leq K-1} \Exp{ \min \left\{ \| \nabla f(\mX_k)\|_F, \| \nabla f(\mX_k)\|_F^2 \right\}}  & = & \mathcal{O}\left( K^{\frac{- 2 \alpha + 2}{3 \alpha - 2}} \right).
        \end{eqnarray*}
        This completes the proof of the theorem.
    \end{proof} 

\newpage
\section{Missing Pseudocode of Algorithm \ref{alg:SGDM_quantile}}\label{sec:missing_pseudocode}

    \begin{algorithm}[H]
    \caption{\algname{SGDM} with Quantile-based Spectral Clipping}\label{alg:SGDM_quantile}
    \begin{algorithmic}[1]
    \REQUIRE Initial point $\mX_0 \in \R^{m \times n}$,
    momentum matrix $\mB_{-1}=0$,
    momentum parameter $\beta\in[0,1)$,
    quantile $q\in(0,1]$,
    window size $w\ge1$,
    initial threshold $\tau_{-1}=0$.
    \STATE Initialize multiset $\gS_{-1}\gets\emptyset$.
    \FOR{$k=0,1,\dots,K-1$}
        \STATE Compute gradient $\mG_k$.
        \STATE $\mB_k \gets \beta \mB_{k-1} + (1-\beta)\,\gC_{\tau_{k-1}}(\mG_k)$.
        \STATE $\mX_{k+1} \gets \mX_k - \eta \mB_k$.
        \STATE $\gS_k \gets \gS_{k-1} \cup \sigma_{\max}(\mG_k)$.
        \IF{$k \ge w$}
            \STATE $\gS_k \gets \gS_k \setminus \sigma_{\max}(\mG_{k-w - 1})$. 
        \ENDIF
        \STATE $\tau_k \gets \gQ_q(\gS_k)$.
    \ENDFOR
    \end{algorithmic}
    \end{algorithm}


    \section{Missing Figures}\label{sec:missing_figures}
    \begin{figure*}[h]
        \centering
        \begin{subfigure}[t]{0.48\textwidth}
            \centering
            \includegraphics[width=\linewidth]{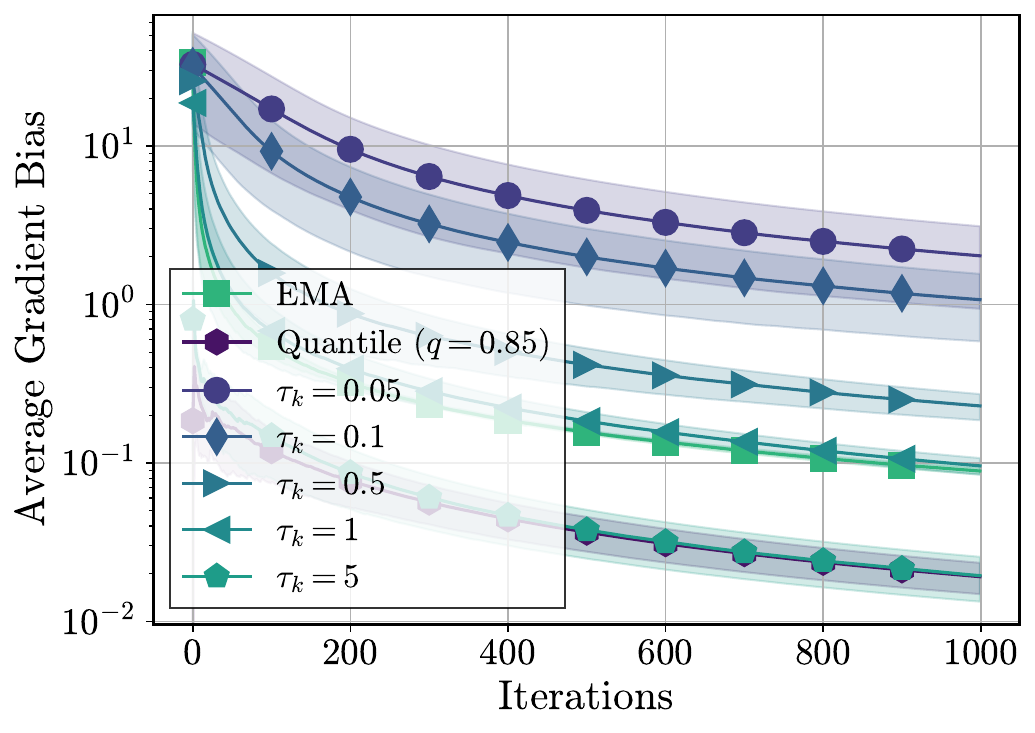}
            \caption{Comparison of $\tau_k$}\label{fig:trace_reg}
        \end{subfigure}
        \hfill
        \begin{subfigure}[t]{0.48\textwidth}
            \centering
            \includegraphics[width=\linewidth]{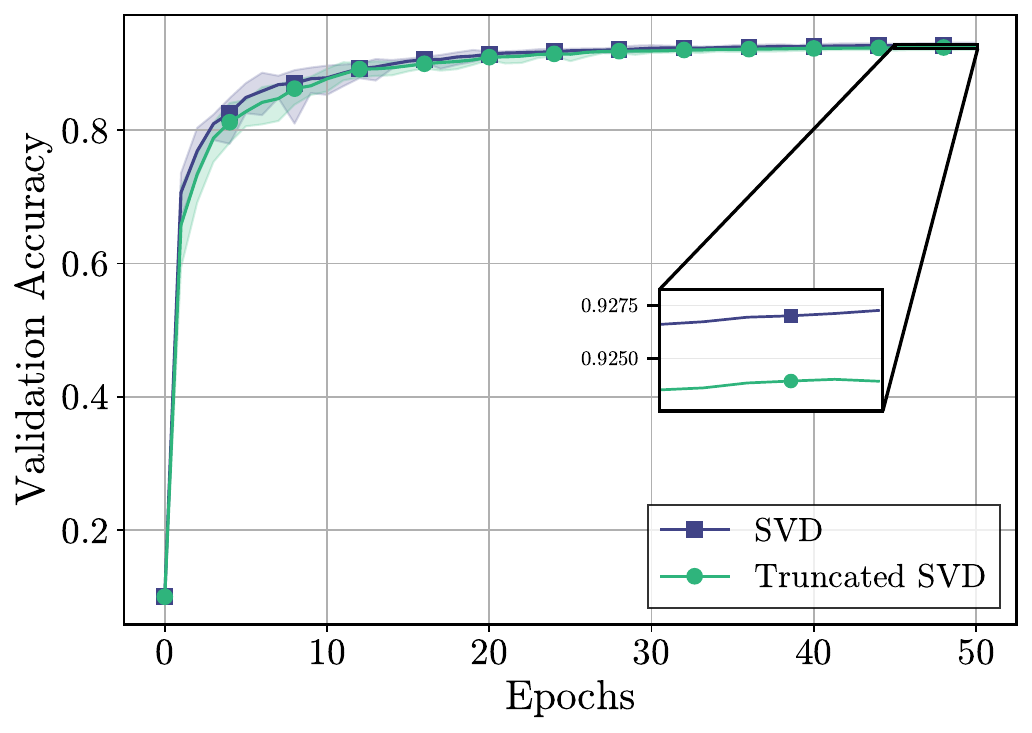}
            \caption{Validation accuracy of spectral clipping with truncated SVD vs.\ full SVD}\label{fig:SVDvstruncatedSVD_accuracy}
        \end{subfigure}
    
        \caption{\small 
        In Figure \ref{fig:trace_reg}, we compare the performance of different spectral clipping strategies on trace regression problem~\eqref{eq:trace_reg}. In Figure \ref{fig:SVDvstruncatedSVD_accuracy} we compare the performance of randomized truncated SVD (Algorithm \ref{alg:spectral_clipping_rsvd}) with SVD for spectral clipping on CV problem.
        }
    \end{figure*}

    \begin{figure}[h]
        \centering
        \begin{subfigure}[t]{0.3\linewidth}
            \centering
                \includegraphics[width=\linewidth]{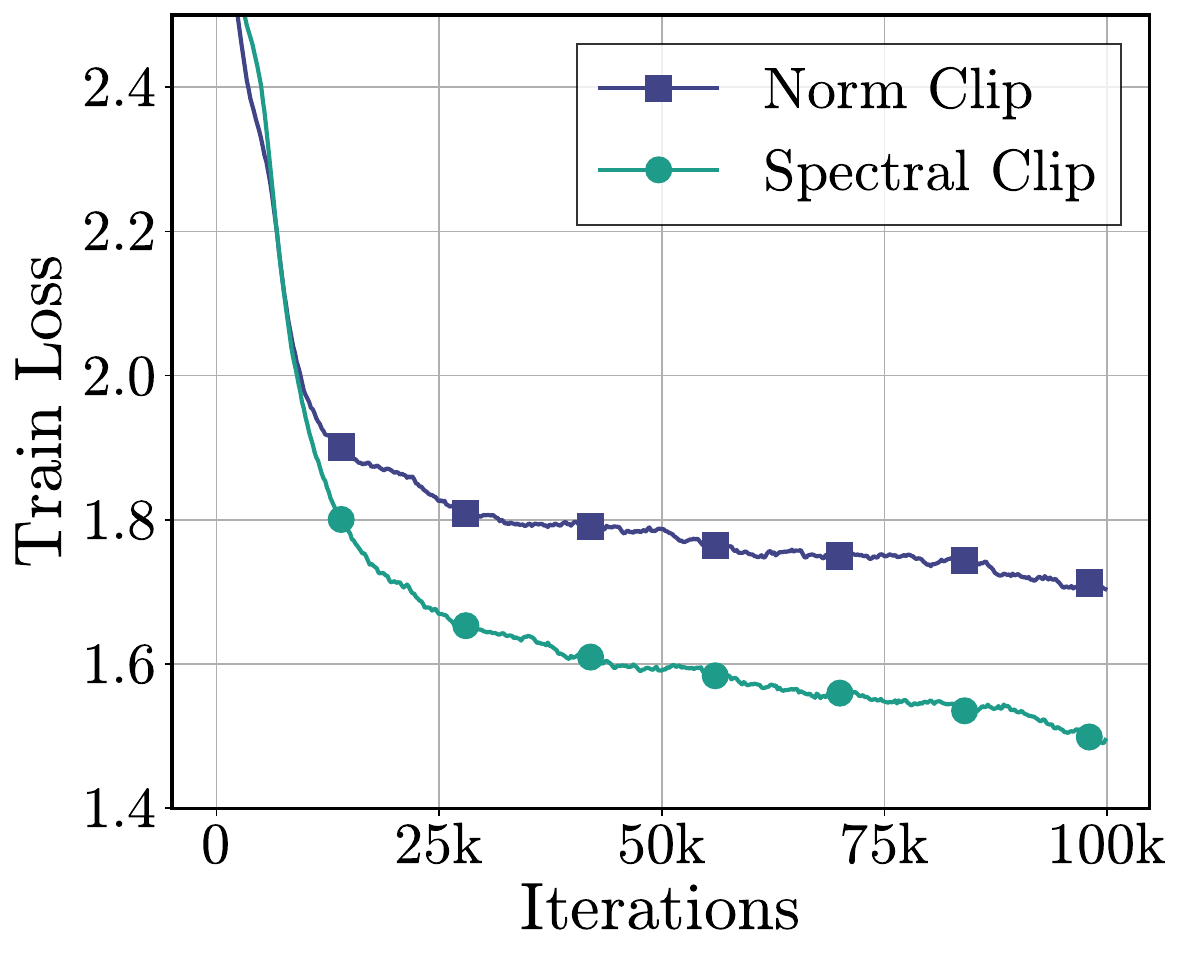}
            \caption{\small \algname{SGDM}}
            \label{fig:nanogpt_sgdm}
        \end{subfigure}
        \hfill
        \begin{subfigure}[t]{0.3\linewidth}
            \centering
            \includegraphics[width=\linewidth]{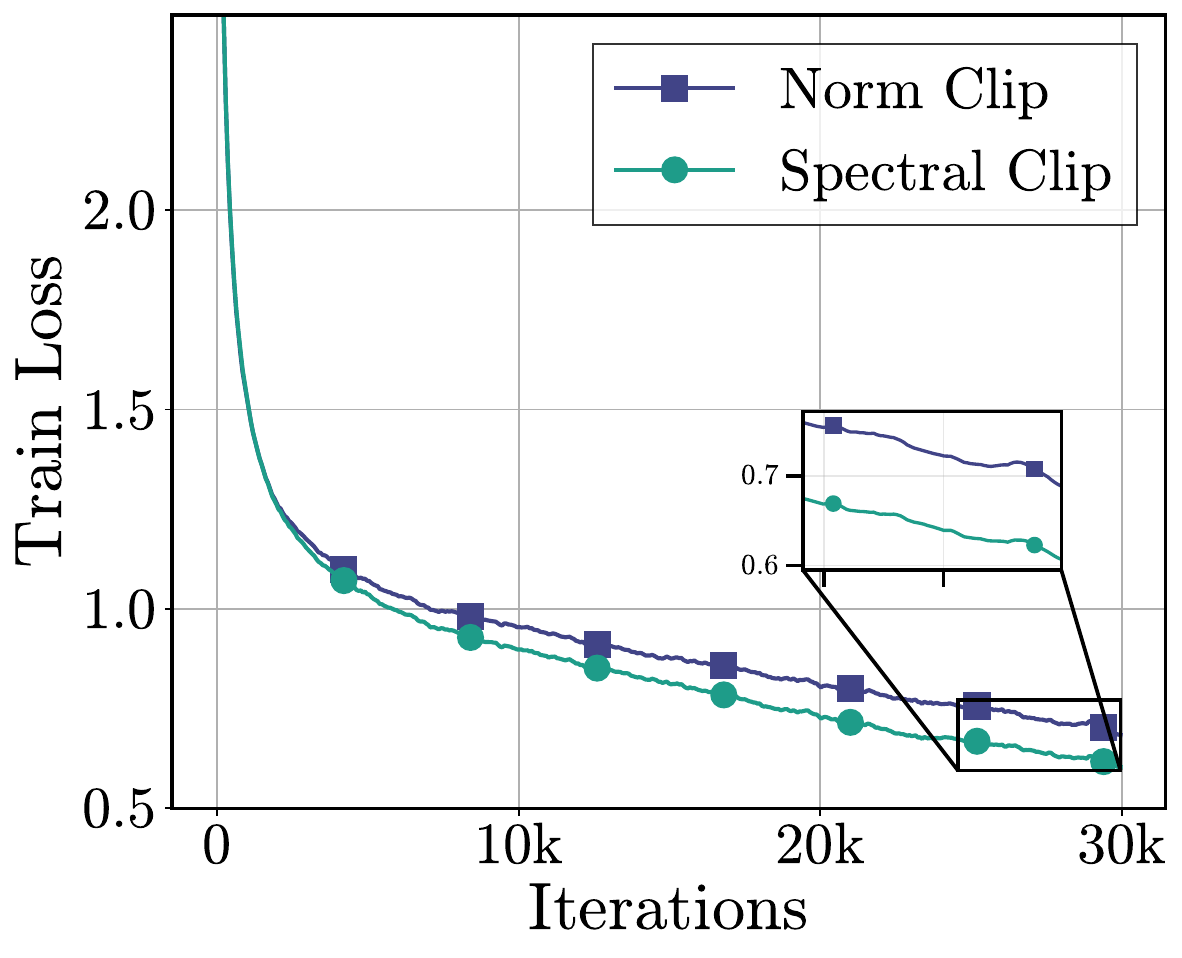}
            \caption{\small \algname{Muon}}
            \label{fig:nanogpt_muon}
        \end{subfigure}
        \hfill
        \begin{subfigure}[t]{0.33\linewidth}
            \centering
            \includegraphics[width=\linewidth]{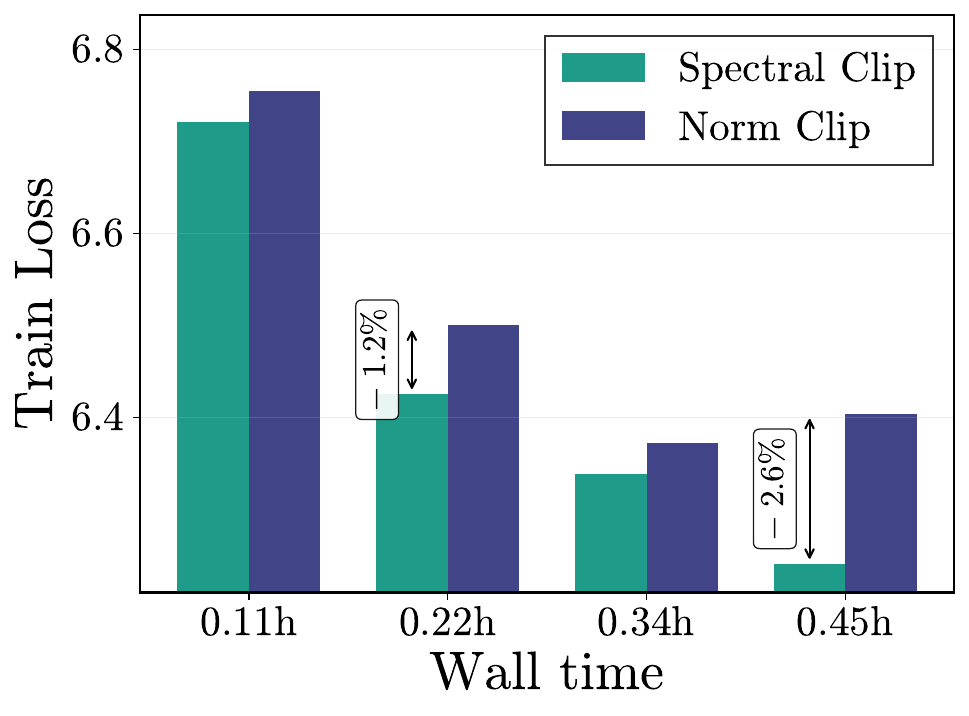}
            \caption{\small \algname{SGDM} (wall-clock time)}
            \label{fig:gpt2_sgdm_wall_time}
        \end{subfigure}

        \caption{\small In Figure~\ref{fig:nanogpt_sgdm} and Figure~\ref{fig:nanogpt_muon}, we show training loss of NanoGPT on the Shakespeare dataset for different optimizers. In Figure~\ref{fig:gpt2_sgdm_wall_time} we show the wall time comparison  of GPT-2 pretraining experiment.}
        \label{fig:nanogpt}
        \vspace{-.3cm}
    \end{figure}

\section{Hyperparameter Tuning Details}\label{sec:missing_experiments}

    All experiments were conducted on a system equipped with $4 \times$ NVIDIA A100 GPUs, and hyperparameters were tuned using Bayesian optimization via Weights \& Biases~\citep{wandb}.  

    \begin{table}[h]
    \centering
    \small
    \caption{\small Best hyperparameters obtained for the MLP experiments.}
    \label{tab:mlp_best_hparams}
    \begin{tabular}{lll}
    \toprule
    Method & Learning rate $\eta$ & Clipping parameter \\
    \midrule
    \algname{SGDM} Norm Clip          & $3.48\!\times\!10^{-3}$ & \texttt{max\_norm}$=3.28$ \\
    \algname{SGDM} Spectral Const     & $3.56\!\times\!10^{-3}$ & $\tau=0.891$ \\
    \algname{SGDM} Spectral EMA       & $1.39\!\times\!10^{-3}$ & --- \\
    \algname{SGDM} Spectral Quantile  & $8.54\!\times\!10^{-3}$ & $q=0.754$ \\
    \algname{Adam} No Clip            & $1.21\!\times\!10^{-3}$ & --- \\
    \algname{Adam} Norm Clip          & $9.00\!\times\!10^{-4}$ & \texttt{max\_norm}$=1.64$ \\
    \bottomrule
    \end{tabular}
    \end{table}

    \paragraph{MLP.}

    We sweep the learning rate $\eta\in[10^{-4},10^{-1}]$ for every
    method, the norm-clipping threshold and the constant spectral-clipping
    threshold over $[10^{-2},10]$, the EMA coefficient is fixed at
    $\theta=0.9$, and the quantile level is sweeped over $q\in[0.5, 0.99]$
    with a fixed sliding-window of length $w=100$. Each Bayesian sweep
    uses $300$ trials when two hyperparameters are searched and $100$
    trials otherwise. The resulting
    hyperparameters are reported in Table~\ref{tab:mlp_best_hparams}.

    \begin{table}[h]
    \centering
    \small
    \setlength{\tabcolsep}{4pt}
    \caption{\small Best hyperparameters obtained for the CIFAR--10
    experiments. ``filter / head / bias'' denote the three parameter
    groups used by \texttt{airbench--94}.}
    \label{tab:cv_best_hparams}
    \begin{tabular}{lcccccl}
    \toprule
    Method
      & $\eta_{\text{filter}}$ & $\eta_{\text{head}}$ & $\eta_{\text{bias}}$
      & $\beta_{\text{filter}}$ & $\beta_{\text{head}}$
      & Clipping parameter \\
    \midrule
    \algname{SGDM} No Clip
      & $2.12\!\times\!10^{-2}$ & $0.67$ & $5.30\!\times\!10^{-2}$
      & $0.772$ & $0.850$
      & --- \\
    \algname{SGDM} Norm Clip
      & $4.75\!\times\!10^{-2}$ & $0.265$ & $8.08\!\times\!10^{-2}$
      & $0.898$ & $0.862$
      & \texttt{max\_norm}$=15.16$ \\
    \algname{SGDM} Spectral Const
      & $2.62\!\times\!10^{-2}$ & $0.305$ & $3.40\!\times\!10^{-2}$
      & $0.671$ & $0.924$
      & $\tau=3.891$ \\
    \bottomrule
    \end{tabular}
    \end{table}

    \paragraph{CV.}
    

    For the CIFAR-10 experiments with \algname{SGDM}, we searched over
    clipping thresholds in $[10^{-2}, 10^{2}]$ for norm clipping and in
    $[10^{-3}, 10^{2}]$ for spectral clipping. In addition, we searched
    over bias learning rates in $[5 \times 10^{-3}, 0.2]$, filter learning
    rates in $[5 \times 10^{-3}, 0.3]$, head learning rates in
    $[5 \times 10^{-2}, 1.5]$, and filter and head momentum parameters in
    $[0.5, 0.98]$. Each search was run for $200$ iterations.
    The resulting hyperparameters are reported in Table~\ref{tab:cv_best_hparams}.

    \begin{table}[h]
    \centering
    \small
    \caption{\small Best hyperparameters obtained for the NanoGPT experiment.}
    \label{tab:nanogpt_best_hparams}
    \begin{tabular}{lll}
    \toprule
    Method & Learning rate $\eta$ & Clipping parameter \\
    \midrule
    \algname{SGDM} Norm Clip      & $7.52\!\times\!10^{-3}$ & \texttt{max\_norm}$=3.70$ \\
    \algname{SGDM} Spectral Clip  & $5.39\!\times\!10^{-3}$ & $\tau=1.89\!\times\!10^{-2}$ \\
    \algname{Muon} Norm Clip      & $8.74\!\times\!10^{-3}$ & \texttt{max\_norm}$=5.53\!\times\!10^{-2}$ \\
    \algname{Muon} Spectral Clip  & $7.87\!\times\!10^{-3}$ & $\tau=0.101$ \\
    \bottomrule
    \end{tabular}
    \end{table}
    
    \paragraph{NanoGPT.}

    For the NanoGPT experiment with \algname{SGDM} and \algname{Muon} we
    searched the learning rate over $[10^{-4}, 10^{-1}]$ for both clipping
    rules, the norm--clipping threshold over
    $[10^{-2}, 10^{1}]$, and the spectral clipping threshold $\tau$ over
    $[10^{-2}, 10^{1}]$. Each search was run for $300$ iterations. The resulting hyperparameters are reported in Table~\ref{tab:nanogpt_best_hparams}.

    \begin{table}[!ht]
    \centering
    \small
    \caption{\small Best hyperparameters obtained for the GPT-2 pretraining experiment.}
    \label{tab:gpt2_fineweb_best_hparams}
    \begin{tabular}{lll}
    \toprule
    Method & Learning rate $\eta$ & Clipping parameter \\
    \midrule
    \algname{SGDM} Norm Clip      & $7.04\!\times\!10^{-2}$ & \texttt{max\_norm}$=2.18$ \\
    \algname{SGDM} Spectral Clip  & $9.42\!\times\!10^{-2}$ & $\tau=4.34\!\times\!10^{-3}$ \\
    \algname{Adam} Norm Clip      & $1.13\!\times\!10^{-3}$ & \texttt{max\_norm}$=4.56\!\times\!10^{-2}$ \\
    \algname{Adam} Spectral Clip  & $1.89\!\times\!10^{-3}$ & $\tau=7.39\!\times\!10^{-3}$ \\
    \algname{Muon} Norm Clip      & $1.24\!\times\!10^{-2}$ & \texttt{max\_norm}$=2.06$ \\
    \algname{Muon} Spectral Clip  & $1.32\!\times\!10^{-2}$ & $\tau=1.03\!\times\!10^{-2}$ \\
    \bottomrule
    \end{tabular}
    \end{table}

    \paragraph{GPT-2.}

    For the GPT-2 pretraining experiment on FineWeb with
    \algname{SGDM}, \algname{Adam} and \algname{Muon}, we searched the
    learning rate over $[10^{-4}, 10^{-1}]$ for all clipping rules, the
    norm-clipping threshold over $[10^{-2}, 5]$, and the spectral-clipping
    threshold $\tau$ over $[10^{-3}, 10^{-1}]$. Each search was run for
    $50$ iterations. The resulting hyperparameters are reported in
    Table~\ref{tab:gpt2_fineweb_best_hparams}.

\section{Gradient Spectral Profiles under Controlled Token Replacement}
\label{app:singular_values_overview}

This appendix provides additional results supporting the motivation in the main text. 
Figure~\ref{fig:sv_distribution_layer_5_grid} shows the results of introducing bad samples in the batch of good samples for different layers and batch sizes.

\begin{figure}[H]
    \centering
    \includegraphics[width=1\linewidth]{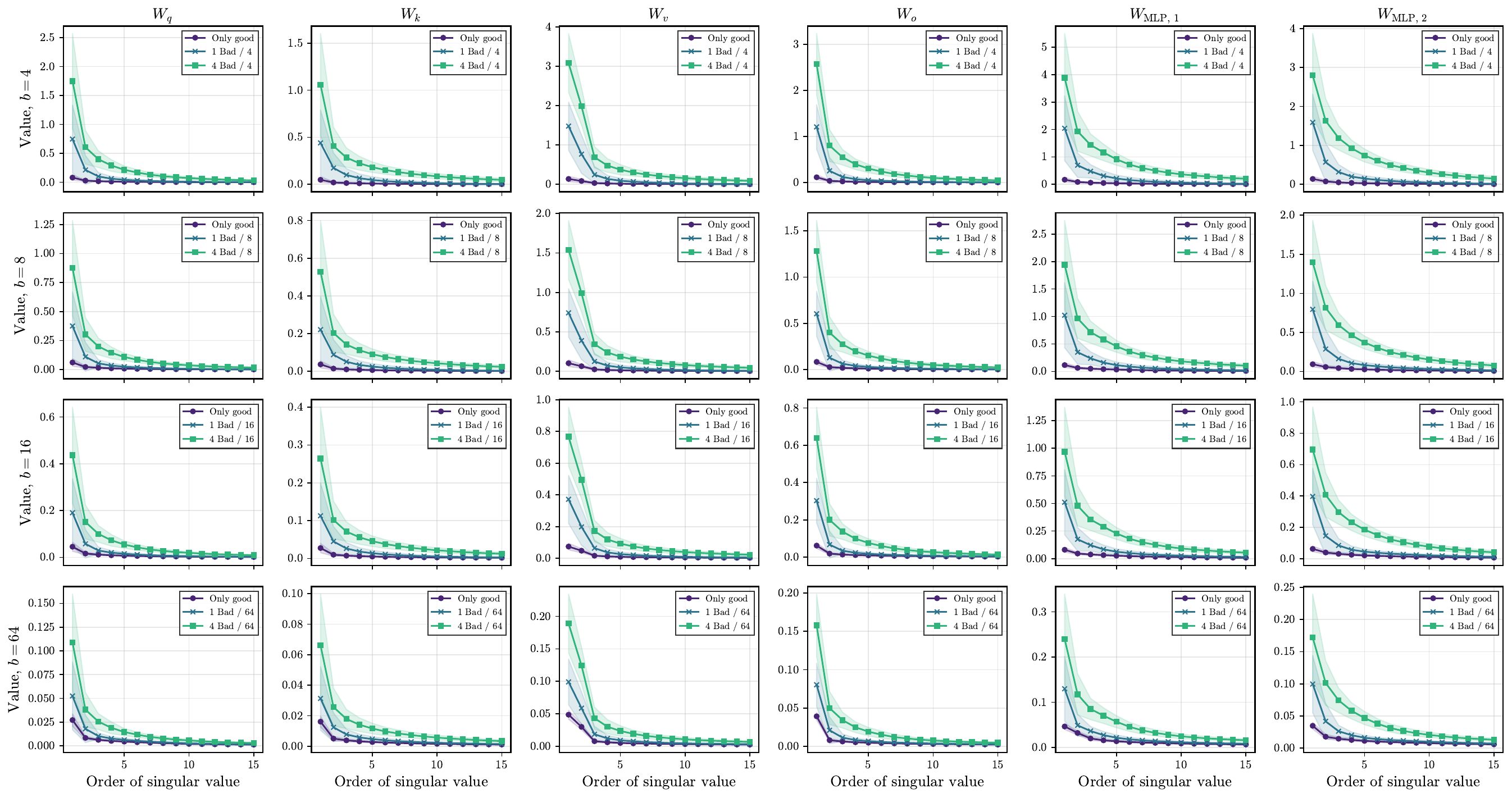}
    \caption{Spectral profile under controlled token replacement of the first feed-forward matrix \(W_{\text{MLP, 1}}\) in layer \textbf{5} of GPT-2 Small for different batch sizes. First, the batch contains only good samples with \(\tau_{\text{good}} = 0.95\), then we replacing one and four of them with bad samples with \(\tau_{\text{bad}} = 10^{-7}\).
    }
    \label{fig:sv_distribution_layer_5_grid}
\end{figure}



\end{document}

%% file: math_commands.tex
\usepackage{amsmath,amsfonts,bm}

\def\eqref#1{(\ref{#1})}

\def\1{\bm{1}}

\def\rva{{\mathbf{a}}}
\def\rvb{{\mathbf{b}}}

\def\vb{{\bm{b}}}

\def\vg{{\bm{g}}}

\def\vu{{\bm{u}}}
\def\vv{{\bm{v}}}

\def\vx{{\bm{x}}}

\def\mA{{\bm{A}}}
\def\mB{{\bm{B}}}

\def\mG{{\bm{G}}}

\def\mU{{\bm{U}}}
\def\mV{{\bm{V}}}

\def\mX{{\bm{X}}}
\def\mY{{\bm{Y}}}

\DeclareMathAlphabet{\mathsfit}{\encodingdefault}{\sfdefault}{m}{sl}
\SetMathAlphabet{\mathsfit}{bold}{\encodingdefault}{\sfdefault}{bx}{n}

\def\gC{{\mathcal{C}}}

\def\gO{{\mathcal{O}}}

\def\gQ{{\mathcal{Q}}}

\def\gS{{\mathcal{S}}}

\newcommand{\E}{\mathbb{E}}

\newcommand{\R}{\mathbb{R}}



%% file: macros.tex
\usepackage{amsmath}
\usepackage{amssymb}
\usepackage{mathtools}
\usepackage{amsthm}
\usepackage{nicefrac}
\usepackage{multirow}
\usepackage{tikz-cd} 
\usepackage{subcaption}
\usepackage{multicol,caption}

\usepackage{wrapfig}
\usepackage{algorithm}

\usepackage[flushleft]{threeparttable} 
\usepackage{colortbl}
\definecolor{bgcolor}{rgb}{0.8,1,1}
\definecolor{bgcolor2}{rgb}{0.8,1,0.8}
\definecolor{niceblue}{rgb}{0.0,0.19,0.56}

\usepackage{hyperref}
\hypersetup{colorlinks,linkcolor={blue},citecolor={blue},urlcolor={blue}}

\usepackage{pifont}
\definecolor{PineGreen}{RGB}{0,110,51}
\definecolor{BrickRed}{RGB}{143,20,2}

\usepackage{thmtools}

\newcommand{\Exp}[1]{\mathbb{E} \left[ #1\right]}

\usepackage[colorinlistoftodos,bordercolor=orange,backgroundcolor=orange!20,linecolor=orange,textsize=scriptsize]{todonotes}

\definecolor{shadecolor}{gray}{0.9}
\declaretheoremstyle[
headfont=\normalfont\bfseries,
notefont=\mdseries, notebraces={(}{)},
bodyfont=\normalfont,
postheadspace=0.5em,
spaceabove=1pt,
mdframed={
  skipabove=8pt,
  skipbelow=8pt,
  hidealllines=true,
  backgroundcolor={shadecolor},
  innerleftmargin=4pt,
  innerrightmargin=4pt,
  innertopmargin=10pt,    
  innerbottommargin=6pt
}
]{shaded}

\declaretheorem[style=shaded,within=section]{definition}
\declaretheorem[style=shaded,sibling=definition]{theorem}

\declaretheorem[style=shaded,sibling=definition]{assumption}

\declaretheorem[style=shaded,sibling=definition]{lemma}

\usepackage[capitalize,noabbrev]{cleveref}

\usepackage{xspace}
\usepackage{xspace}

\newcommand{\algname}[1]{{\sf  #1}\xspace}

\newcommand{\la}{\left\langle}
\newcommand{\ra}{\right\rangle}

\newcommand{\eqdef}{:=}

\DeclareMathOperator{\clamp}{clamp}

%% file: references.bib
@article{achiam2023gpt,
  title={Gpt-4 technical report},
  author={Achiam, Josh and Adler, Steven and Agarwal, Sandhini and Ahmad, Lama and Akkaya, Ilge and Aleman, Florencia Leoni and Almeida, Diogo and Altenschmidt, Janko and Altman, Sam and Anadkat, Shyamal and others},
  journal={arXiv preprint arXiv:2303.08774},
  year={2023}
}

@article{amsel2025polar,
  title={The polar express: Optimal matrix sign methods and their application to the muon algorithm},
  author={Amsel, Noah and Persson, David and Musco, Christopher and Gower, Robert M},
  journal={arXiv preprint arXiv:2505.16932},
  year={2025}
}

@article{an2025asgo,
  title={Asgo: Adaptive structured gradient optimization},
  author={An, Kang and Liu, Yuxing and Pan, Rui and Ren, Yi and Ma, Shiqian and Goldfarb, Donald and Zhang, Tong},
  journal={arXiv preprint arXiv:2503.20762},
  year={2025}
}

@article{ahn2023linear,
  title={Linear attention is (maybe) all you need (to understand transformer optimization)},
  author={Ahn, Kwangjun and Cheng, Xiang and Song, Minhak and Yun, Chulhee and Jadbabaie, Ali and Sra, Suvrit},
  journal={arXiv preprint arXiv:2310.01082},
  year={2023}
}

@inproceedings{bjorck2021understanding,
  title={Understanding decoupled and early weight decay},
  author={Bjorck, Johan and Weinberger, Kilian Q and Gomes, Carla},
  booktitle={Proceedings of the AAAI Conference on Artificial Intelligence},
  volume={35},
  number={8},
  pages={6777--6785},
  year={2021}
}

@inproceedings{battash2024revisiting,
  title={Revisiting the noise model of stochastic gradient descent},
  author={Battash, Barak and Wolf, Lior and Lindenbaum, Ofir},
  booktitle={International Conference on Artificial Intelligence and Statistics},
  pages={4780--4788},
  year={2024},
  organization={PMLR}
}

@inproceedings{brock2021high,
  title={High-performance large-scale image recognition without normalization},
  author={Brock, Andy and De, Soham and Smith, Samuel L and Simonyan, Karen},
  booktitle={International conference on machine learning},
  pages={1059--1071},
  year={2021},
  organization={PMLR}
}

@article{chang2025convergence,
  title={On the Convergence of Muon and Beyond},
  author={Chang, Da and Liu, Yongxiang and Yuan, Ganzhao},
  journal={arXiv preprint arXiv:2509.15816},
  year={2025}
}

@article{chezhegov2024clipping,
  title={Clipping Improves Adam-Norm and AdaGrad-Norm when the Noise Is Heavy-Tailed},
  author={Chezhegov, Savelii and Klyukin, Yaroslav and Semenov, Andrei and Beznosikov, Aleksandr and Gasnikov, Alexander and Horv{\'a}th, Samuel and Tak{\'a}{\v{c}}, Martin and Gorbunov, Eduard},
  journal={arXiv preprint arXiv:2406.04443},
  year={2024}
}

@misc{choudhury2026muonnesterovmomentumheavytailed,
      title={Muon with Nesterov Momentum: Heavy-Tailed Noise and (Randomized) Inexact Polar Decomposition}, 
      author={Sayantan Choudhury and Xiaoran Cheng and Martin Takáč and Sen Na and Mladen Kolar},
      year={2026},
      eprint={2605.06884},
      archivePrefix={arXiv},
      primaryClass={math.OC},
      url={https://arxiv.org/abs/2605.06884}, 
}

@article{dekel2012optimal,
  title={Optimal distributed online prediction using mini-batches},
  author={Dekel, Ofer and Gilad-Bachrach, Ran and Shamir, Ohad and Xiao, Lin},
  journal={The Journal of Machine Learning Research},
  volume={13},
  number={1},
  pages={165--202},
  year={2012},
  publisher={JMLR. org}
}

@article{dongarra2018singular,
  title={The singular value decomposition: Anatomy of optimizing an algorithm for extreme scale},
  author={Dongarra, Jack and Gates, Mark and Haidar, Azzam and Kurzak, Jakub and Luszczek, Piotr and Tomov, Stanimire and Yamazaki, Ichitaro},
  journal={SIAM review},
  volume={60},
  number={4},
  pages={808--865},
  year={2018},
  publisher={SIAM}
}

@article{eckart1936approximation,
  title={The approximation of one matrix by another of lower rank},
  author={Eckart, Carl and Young, Gale},
  journal={Psychometrika},
  volume={1},
  number={3},
  pages={211--218},
  year={1936},
  publisher={Springer-Verlag}
}

@article{guo2025deepseek,
  title={Deepseek-r1: Incentivizing reasoning capability in llms via reinforcement learning},
  author={Guo, Daya and Yang, Dejian and Zhang, Haowei and Song, Junxiao and Zhang, Ruoyu and Xu, Runxin and Zhu, Qihao and Ma, Shirong and Wang, Peiyi and Bi, Xiao and others},
  journal={arXiv preprint arXiv:2501.12948},
  year={2025}
}

@inproceedings{gupta2018shampoo,
  title={Shampoo: Preconditioned stochastic tensor optimization},
  author={Gupta, Vineet and Koren, Tomer and Singer, Yoram},
  booktitle={International Conference on Machine Learning},
  pages={1842--1850},
  year={2018},
  organization={PMLR}
}

@article{ghadimi2013stochastic,
  title={Stochastic first-and zeroth-order methods for nonconvex stochastic programming},
  author={Ghadimi, Saeed and Lan, Guanghui},
  journal={SIAM journal on optimization},
  volume={23},
  number={4},
  pages={2341--2368},
  year={2013},
  publisher={SIAM}
}

@inproceedings{gorbunov2020unified,
  title={A unified theory of SGD: Variance reduction, sampling, quantization and coordinate descent},
  author={Gorbunov, Eduard and Hanzely, Filip and Richt{\'a}rik, Peter},
  booktitle={International Conference on Artificial Intelligence and Statistics},
  pages={680--690},
  year={2020},
  organization={PMLR}
}

@article{gorbunov2020stochastic,
  title={Stochastic optimization with heavy-tailed noise via accelerated gradient clipping},
  author={Gorbunov, Eduard and Danilova, Marina and Gasnikov, Alexander},
  journal={Advances in Neural Information Processing Systems},
  volume={33},
  pages={15042--15053},
  year={2020}
}

@inproceedings{garg2021proximal,
  title={On proximal policy optimization’s heavy-tailed gradients},
  author={Garg, Saurabh and Zhanson, Joshua and Parisotto, Emilio and Prasad, Adarsh and Kolter, Zico and Lipton, Zachary and Balakrishnan, Sivaraman and Salakhutdinov, Ruslan and Ravikumar, Pradeep},
  booktitle={International Conference on Machine Learning},
  pages={3610--3619},
  year={2021},
  organization={PMLR}
}

@misc{he2015deepresiduallearningimage,
      title={Deep Residual Learning for Image Recognition}, 
      author={Kaiming He and Xiangyu Zhang and Shaoqing Ren and Jian Sun},
      year={2015},
      eprint={1512.03385},
      archivePrefix={arXiv},
      primaryClass={cs.CV},
      url={https://arxiv.org/abs/1512.03385}, 
}

@inproceedings{he2016deep,
  title={Deep residual learning for image recognition},
  author={He, Kaiming and Zhang, Xiangyu and Ren, Shaoqing and Sun, Jian},
  booktitle={Proceedings of the IEEE conference on computer vision and pattern recognition},
  pages={770--778},
  year={2016}
}

@article{higham1986computing,
  title={Computing the polar decomposition—with applications},
  author={Higham, Nicholas J},
  journal={SIAM Journal on Scientific and Statistical Computing},
  volume={7},
  number={4},
  pages={1160--1174},
  year={1986},
  publisher={SIAM}
}

@article{halko2011finding,
  title={Finding structure with randomness: Probabilistic algorithms for constructing approximate matrix decompositions},
  author={Halko, Nathan and Martinsson, Per-Gunnar and Tropp, Joel A},
  journal={SIAM review},
  volume={53},
  number={2},
  pages={217--288},
  year={2011},
  publisher={SIAM}
}

@article{hubler2024gradient,
  title={From gradient clipping to normalization for heavy tailed sgd},
  author={H{\"u}bler, Florian and Fatkhullin, Ilyas and He, Niao},
  journal={arXiv preprint arXiv:2410.13849},
  year={2024}
}

@misc{jordan2024muon,
  author       = {Keller Jordan and Yuchen Jin and Vlado Boza and Jiacheng You and
                  Franz Cesista and Laker Newhouse and Jeremy Bernstein},
  title        = {Muon: An optimizer for hidden layers in neural networks},
  year         = {2024},
  url          = {https://kellerjordan.github.io/posts/muon/}
}

@article{jordan2024airbench,
  title={CIFAR-10 Airbench: Fast Neural Network Training},
  author={Jordan, Keller},
  journal={arXiv preprint arXiv:2404.00498},
  year={2024}
}

@misc{Karpathy2022,
  author = {Andrej Karpathy},
  title = {\text{NanoGPT}},
  year = {2022},
  publisher = {GitHub},
  journal = {GitHub repository},
  howpublished = {\url{https://github.com/karpathy/nanoGPT}},
  commit = {325be85d9be8c81b436728a420e85796c57dba7e}
}

@article{kingma2014adam,
  title={Adam: A method for stochastic optimization},
  author={Kingma, Diederik P},
  journal={arXiv preprint arXiv:1412.6980},
  year={2014}
}

@inproceedings{koloskova2023revisiting,
  title={Revisiting gradient clipping: Stochastic bias and tight convergence guarantees},
  author={Koloskova, Anastasia and Hendrikx, Hadrien and Stich, Sebastian U},
  booktitle={International Conference on Machine Learning},
  pages={17343--17363},
  year={2023},
  organization={PMLR}
}

@inproceedings{kadri2020partial,
  title={Partial trace regression and low-rank kraus decomposition},
  author={Kadri, Hachem and Ayache, St{\'e}phane and Huusari, Riikka and Rakotomamonjy, Alain and Liva, Ralaivola},
  booktitle={International Conference on Machine Learning},
  pages={5031--5041},
  year={2020},
  organization={PMLR}
}

@techreport{krizhevsky2009learning,
  title={Learning Multiple Layers of Features from Tiny Images},
  author={Krizhevsky, Alex and Hinton, Geoffrey},
  institution={University of Toronto},
  year={2009}
}

@article{mikolov2012statistical,
  title={Statistical language models based on neural networks},
  author={Mikolov, Tom{\'a}{\v{s}} and others},
  journal={Presentation at Google, Mountain View, 2nd April},
  volume={80},
  number={26},
  year={2012}
}

@article{loshchilov2017decoupled,
  title={Decoupled weight decay regularization},
  author={Loshchilov, Ilya and Hutter, Frank},
  journal={arXiv preprint arXiv:1711.05101},
  year={2017}
}

@article{lecun2015deep,
  title={Deep learning},
  author={LeCun, Yann and Bengio, Yoshua and Hinton, Geoffrey},
  journal={nature},
  volume={521},
  number={7553},
  pages={436--444},
  year={2015},
  publisher={Nature Publishing Group UK London}
}

@article{lan2012optimal,
  title={An optimal method for stochastic composite optimization},
  author={Lan, Guanghui},
  journal={Mathematical Programming},
  volume={133},
  number={1},
  pages={365--397},
  year={2012},
  publisher={Springer}
}

@article{liu2020improved,
  title={An improved analysis of stochastic gradient descent with momentum},
  author={Liu, Yanli and Gao, Yuan and Yin, Wotao},
  journal={Advances in Neural Information Processing Systems},
  volume={33},
  pages={18261--18271},
  year={2020}
}

@article{mohammadi2015estimating,
  title={On estimating the tail index and the spectral measure of multivariate $\alpha$-stable distributions},
  author={Mohammadi, Mohammad and Mohammadpour, Adel and Ogata, Hiroaki},
  journal={Metrika},
  volume={78},
  number={5},
  pages={549--561},
  year={2015},
  publisher={Springer}
}

@book{nesterov2004introductory,
  title     = {Introductory Lectures on Convex Optimization: A Basic Course},
  author    = {Nesterov, Yurii},
  series    = {Applied Optimization},
  volume    = {87},
  publisher = {Springer},
  year      = {2004}
}

@article{penedo2024fineweb,
  title={The {FineWeb} Datasets: Decanting the Web for the Finest Text Data at Scale},
  author={Penedo, Guilherme and Kydl{\'\i}{\v{c}}ek, Hynek and von Werra, Leandro and Wolf, Thomas and others},
  journal={arXiv preprint arXiv:2406.17557},
  year={2024}
}

@article{pethick2025generalized,
  title={Generalized Gradient Norm Clipping \& Non-Euclidean $(L\_0, L\_1) $-Smoothness},
  author={Pethick, Thomas and Xie, Wanyun and Erdogan, Mete and Antonakopoulos, Kimon and Silveti-Falls, Tony and Cevher, Volkan},
  journal={Advances in Neural Information Processing Systems (NeurIPS)},
  year={2025}
}

@article{pethick2025training,
  title={Training deep learning models with norm-constrained lmos},
  author={Pethick, Thomas and Xie, Wanyun and Antonakopoulos, Kimon and Zhu, Zhenyu and Silveti-Falls, Antonio and Cevher, Volkan},
  journal={arXiv preprint arXiv:2502.07529},
  year={2025}
}

@inproceedings{pascanu2013difficulty,
  title={On the difficulty of training recurrent neural networks},
  author={Pascanu, Razvan and Mikolov, Tomas and Bengio, Yoshua},
  booktitle={International conference on machine learning},
  pages={1310--1318},
  year={2013},
  organization={Pmlr}
}

@article{polyak1964some,
  title={Some methods of speeding up the convergence of iteration methods},
  author={Polyak, Boris T},
  journal={Ussr computational mathematics and mathematical physics},
  volume={4},
  number={5},
  pages={1--17},
  year={1964},
  publisher={Elsevier}
}

@inproceedings{qian2021understanding,
  title={Understanding gradient clipping in incremental gradient methods},
  author={Qian, Jiang and Wu, Yuren and Zhuang, Bojin and Wang, Shaojun and Xiao, Jing},
  booktitle={International Conference on Artificial Intelligence and Statistics},
  pages={1504--1512},
  year={2021},
  organization={PMLR}
}

@article{robbins1951stochastic,
  title={A stochastic approximation method},
  author={Robbins, Herbert and Monro, Sutton},
  journal={The annals of mathematical statistics},
  pages={400--407},
  year={1951},
  publisher={JSTOR}
}

@article{radford2019language,
  title={Language models are unsupervised multitask learners},
  author={Radford, Alec and Wu, Jeffrey and Child, Rewon and Luan, David and Amodei, Dario and Sutskever, Ilya and others},
  journal={OpenAI blog},
  volume={1},
  number={8},
  pages={9},
  year={2019}
}

@article{shen2025convergence,
  title={On the convergence analysis of muon},
  author={Shen, Wei and Huang, Ruichuan and Huang, Minhui and Shen, Cong and Zhang, Jiawei},
  journal={arXiv preprint arXiv:2505.23737},
  year={2025}
}

@article{shah2025practical,
  title={Practical efficiency of muon for pretraining},
  author={Shah, Ishaan and Polloreno, Anthony M and Stratos, Karl and Monk, Philip and Chaluvaraju, Adarsh and Hojel, Andrew and Ma, Andrew and Thomas, Anil and Tanwer, Ashish and Shah, Darsh J and others},
  journal={arXiv preprint arXiv:2505.02222},
  year={2025}
}

@inproceedings{simsekli2019tail,
  title={A tail-index analysis of stochastic gradient noise in deep neural networks},
  author={Simsekli, Umut and Sagun, Levent and Gurbuzbalaban, Mert},
  booktitle={International Conference on Machine Learning},
  pages={5827--5837},
  year={2019},
  organization={PMLR}
}

@article{sutskever2014sequence,
  title={Sequence to sequence learning with neural networks},
  author={Sutskever, Ilya and Vinyals, Oriol and Le, Quoc V},
  journal={Advances in neural information processing systems},
  volume={27},
  year={2014}
}

@article{stewart1993early,
  title={On the early history of the singular value decomposition},
  author={Stewart, Gilbert W},
  journal={SIAM review},
  volume={35},
  number={4},
  pages={551--566},
  year={1993},
  publisher={SIAM}
}

@article{su2025galore,
  title={Galore 2: Large-scale llm pre-training by gradient low-rank projection},
  author={Su, DiJia and Gu, Andrew and Xu, Jane and Tian, Yuandong and Zhao, Jiawei},
  journal={arXiv preprint arXiv:2504.20437},
  year={2025}
}

@article{tieleman2012lecture,
  title={Lecture 6.5-rmsprop: Divide the gradient by a running average of its recent magnitude},
  author={Tieleman, Tijmen},
  journal={COURSERA: Neural networks for machine learning},
  volume={4},
  number={2},
  pages={26},
  year={2012}
}

@article{vaswani2017attention,
  title={Attention is all you need},
  author={Vaswani, Ashish and Shazeer, Noam and Parmar, Niki and Uszkoreit, Jakob and Jones, Llion and Gomez, Aidan N and Kaiser, {\L}ukasz and Polosukhin, Illia},
  journal={Advances in neural information processing systems},
  volume={30},
  year={2017}
}

@article{vyas2024soap,
  title={Soap: Improving and stabilizing shampoo using adam},
  author={Vyas, Nikhil and Morwani, Depen and Zhao, Rosie and Kwun, Mujin and Shapira, Itai and Brandfonbrener, David and Janson, Lucas and Kakade, Sham},
  journal={arXiv preprint arXiv:2409.11321},
  year={2024}
}

@misc{wandb,
    title = {Experiment Tracking with Weights and Biases},
    year = {2020},
    note = {Software available from wandb.com},
    url={https://www.wandb.com/},
    author = {Biewald, Lukas},
}

@article{zeiler2012adadelta,
  title={Adadelta: an adaptive learning rate method},
  author={Zeiler, Matthew D},
  journal={arXiv preprint arXiv:1212.5701},
  year={2012}
}

@article{zhang2020improved,
  title={Improved analysis of clipping algorithms for non-convex optimization},
  author={Zhang, Bohang and Jin, Jikai and Fang, Cong and Wang, Liwei},
  journal={Advances in Neural Information Processing Systems},
  volume={33},
  pages={15511--15521},
  year={2020}
}

@misc{zhang2020adaptivemethodsgoodattention,
      title={Why are Adaptive Methods Good for Attention Models?}, 
      author={Jingzhao Zhang and Sai Praneeth Karimireddy and Andreas Veit and Seungyeon Kim and Sashank J Reddi and Sanjiv Kumar and Suvrit Sra},
      year={2020},
      eprint={1912.03194},
      archivePrefix={arXiv},
      primaryClass={math.OC},
      url={https://arxiv.org/abs/1912.03194}, 
}

@article{zhao2024galore,
  title={Galore: Memory-efficient llm training by gradient low-rank projection},
  author={Zhao, Jiawei and Zhang, Zhenyu and Chen, Beidi and Wang, Zhangyang and Anandkumar, Anima and Tian, Yuandong},
  journal={arXiv preprint arXiv:2403.03507},
  year={2024}
}

@article{zhang2015singular,
  title={The singular value decomposition, applications and beyond},
  author={Zhang, Zhihua},
  journal={arXiv preprint arXiv:1510.08532},
  year={2015}
}


%% file: references_martin.bib
@inproceedings{takavc2013mini,
  title={Mini-batch primal and dual methods for {SVMs}},
  author={Tak{\'a}{\v{c}}, Martin and Bijral, Avleen and Richt{\'a}rik, Peter and Srebro, Nathan},
  booktitle={In 30th International Conference on Machine Learning, ICML 2013},
  year={2013}
}

@inproceedings{berahas2016multi,
  title={A Multi-Batch {L-BFGS} Method for Machine Learning},
  author={Berahas, Albert S and Nocedal, Jorge and Tak{\'a}{\v{c}}, Martin},
  booktitle={The Thirtieth Annual Conference on Neural Information Processing Systems (NIPS)},
  year={2016}
}

@inproceedings{nguyen2017sarah,
  title={{SARAH:} A novel method for machine learning problems using stochastic recursive gradient},
  author={Nguyen, Lam and Liu, Jie and Scheinberg, Katya and Tak{\'a}{\v{c}}, Martin},
  booktitle={In 34th International Conference on Machine Learning, ICML 2017},
  year={2017}
}

@inproceedings{he2017distributed,
  title={Distributed Hessian-Free Optimization for Deep Neural Network.},
  author={He, Xi and Mudigere, Dheevatsa and Smelyanskiy, Mikhail and Tak{\'a}{\v{c}}, Martin},
  booktitle={AAAI Workshops},
  year={2017}
}

@article{schaipp2023momo,
  title={Momo: Momentum models for adaptive learning rates},
  author={Schaipp, Fabian and Ohana, Ruben and Eickenberg, Michael and Defazio, Aaron and Gower, Robert M},
  journal={arXiv preprint arXiv:2305.07583},
  year={2023}
}

@article{li2022sp2,
  title={{SP2:} A second order stochastic Polyak method},
  author={Li, Shuang and Swartworth, William J and Tak{\'a}{\v{c}}, Martin and Needell, Deanna and Gower, Robert M},
  journal={ICLR 2023},
  year={2022}
}

@inproceedings{choudhury2024remove,
  title={Remove that Square Root: A New Efficient Scale-Invariant Version of AdaGrad},
  author={Choudhury, Sayantan and Tupitsa, Nazarii and Loizou, Nicolas and Horvath, Samuel and ak{\'a}{\v{c}}, Martin and Gorbunov, Eduard},
  booktitle={NeurIPS},
  year={2024}
}

@article{mishchenko2023prodigy,
  title={Prodigy: An expeditiously adaptive parameter-free learner},
  author={Mishchenko, Konstantin and Defazio, Aaron},
  journal={arXiv preprint arXiv:2306.06101},
  year={2023}
}

@article{shi2023ai,
  title={{AI-SARAH:} Adaptive and implicit stochastic recursive gradient methods},
  author={Shi, Zheng and Sadiev, Abdurakhmon and Loizou, Nicolas and Richt{\'a}rik, Peter and Tak{\'a}{\v{c}}, Martin},
  journal={Transactions on Machine Learning Research},
  year={2023}
}
